\documentclass{article}
\usepackage{microtype}
\usepackage{graphicx}
\usepackage{subcaption}
\usepackage{booktabs} 
\usepackage{enumitem}

\usepackage{hyperref}

\usepackage[preprint]{icml2026}

\usepackage{multirow} 
\usepackage{booktabs}
\usepackage{makecell}
\usepackage{array}
\usepackage{amsmath}
\usepackage{amssymb}
\usepackage{mathtools}
\usepackage{amsthm}
\usepackage{xcolor} 
\usepackage{float}
\theoremstyle{plain}
\newtheorem{theorem}{Theorem}[section]
\newtheorem{proposition}[theorem]{Proposition}

\theoremstyle{definition}
\newtheorem{definition}[theorem]{Definition}

\theoremstyle{remark}

\usepackage[textsize=tiny]{todonotes}

\begin{document}

\twocolumn[
  \icmltitle{FOCAL-Attention for Heterogeneous Multi-Label Prediction}

  \icmlsetsymbol{equal}{*}

  \begin{icmlauthorlist}
    \icmlauthor{Chenghao Zhang}{equal,aff1,aff2}
    \icmlauthor{Qingqing Long}{equal,aff1,aff2}
    \icmlauthor{Ludi Wang}{aff1,aff2}
    \icmlauthor{Wenjuan Cui}{aff1,aff2}
    \icmlauthor{Jianjun Yu}{aff1,aff2}
    \icmlauthor{Yi Du}{aff1,aff2}
  \end{icmlauthorlist}

  \icmlaffiliation{aff1}{Computer Network Information Center, Chinese Academy of Sciences, Beijing, China}
  \icmlaffiliation{aff2}{University of Chinese Academy of Sciences, Beijing, China}
  \icmlcorrespondingauthor{Jianjun Yu, Yi Du}{yujj@cnic.ac.cn, duyi@cnic.cn}

  \vskip 0.3in
]

\printAffiliationsAndNotice{}  
\begin{abstract}
Heterogeneous graphs have attracted increasing attention for modeling multi-typed entities and relations in complex real-world systems.
Multi-label node classification on heterogeneous graphs is challenging due to structural heterogeneity and the need to learn shared representations across multiple labels.
Existing methods typically adopt either flexible attention mechanisms or meta-path constrained anchoring, but in heterogeneous multi-label prediction they often suffer from \textbf{semantic dilution} or \textbf{coverage constraint}. Both issues are further amplified under multi-label supervision.
We present a theoretical analysis showing that as heterogeneous neighborhoods expand, the attention mass allocated to task-critical (primary) neighborhoods diminishes, and that meta-path constrained aggregation exhibits a dilemma: too few meta-paths intensify coverage constraint, while too many re-introduce dilution.
To resolve this coverage-anchoring conflict, we propose FOCAL: Fusion Of Coverage and Anchoring Learning, with two components: coverage-oriented attention (COA) for flexible, unconstrained heterogeneous context aggregation, and anchoring-oriented attention (AOA) that restricts aggregation to meta-path-induced primary semantics.
Our theoretical analysis and experimental results further indicates that FOCAL has a better performance than other state-of-the-art methods.

\end{abstract}

\section{Introduction}
Heterogeneous graphs naturally model real-world systems with multiple node and relation types~\cite{hu2020open}, such as user-item-attribute networks in recommendation~\cite{liao2024revgnn}, author-paper-venue networks in academic mining, and entity-relation graphs in biomedical discovery~\cite{ali2025graph}. In many of these applications, a node may simultaneously belong to multiple categories (e.g., a paper spanning multiple topics, a product associated with multiple attributes), making multi-label node classification on heterogeneous graphs a practically important yet challenging problem. While heterogeneous graph neural networks (HGNNs) have achieved strong performance in a variety of downstream tasks~\cite{ju2024comprehensive,yuehyperbolic}, most existing studies primarily focus on single-label settings or do not fully address the unique difficulties arising from heterogeneity coupled with multi-label supervision.

Previous works can be broadly grouped into two lines: \textbf{(1)} The first line develops HGNNs mainly for heterogeneous single-label prediction, leveraging mechanisms such as meta-path based aggregation or attention over typed neighbors~\cite{ji2021heterogeneous1,yang2021interpretable}. Under multi-label supervision, this coverage-anchoring tension is further amplified, making semantic dilution and insufficient coverage more pronounced, as detailed analyzed in Theorem~\ref{lem:multilabel_loss_grad_amp}, \ref{lem:multilabel_amplification_main}.
\textbf{(2)} The second line focuses on multi-label learning on homogeneous graphs~\cite{zhou2021multi,xiao2022semantic}, where type-aware semantics and heterogeneous structural patterns are absent. Directly applying them to heterogeneous graphs often ignores relation specific semantics, which are essential for accurate prediction. 
Only few works attempt to jointly consider heterogeneity and multi-label supervision \cite{gupta2025persona,bei2025correlation}. However, \textbf{existing attempts either simplify multi-label classification into node-label link prediction, or lack an explicit, principled mechanism that separates and coordinates coverage and anchoring roles}.

In real-world applications~\cite{yuehyperbolic,wang2019knowledge}, task-critical semantics often concentrate on a subset of node/edge types or relation patterns, which we refer to as \textbf{primary semantics}~\cite{wang2019kgat,wang2019knowledge}. 
Other types mainly contribute \textbf{secondary semantics}, providing auxiliary context that is usually less predictive for the target labels. 
However, real heterogeneous graphs are typically dominated by abundant secondary types in local neighborhoods, and the presence of high-degree nodes further inflates neighborhood size. As a result, \textbf{(1) the signal from a few primary semantic can be easily overwhelmed by massive contextual information, leading to semantic dilution}-a phenomenon we theoretically characterize under flexible attention based aggregation in Theorem~\ref{lem:softmax_attention_dilution}. Meanwhile, despite being less critical on average, secondary types may still contain rare but decisive nodes in some cases; capturing such rare yet informative context requires a more flexible attention mechanism that can aggregate across types and also distinguish different nodes within the same type. 
These observations reveal an inherent tension: flexible aggregation is needed for broad semantic coverage, yet anchoring primary semantics is necessary to avoid being dominated by noisy or overwhelming context.
Also, when anchoring is enforced via meta-path constrained aggregation, one can face a \textbf{dilemma between (2) coverage constraint} (\textbf{under-attention when too few meta-paths are used)} \textbf{and (3) semantic dilution} \textbf{(over-dispersion when too many meta-paths are included)}, as formalized in Theorem~\ref{lem:meta_path_subspace_dilution_main}. Therefore, a single aggregation mechanism often struggles to achieve both semantic coverage and semantic anchoring in heterogeneous multi-label classification.

To address these challenges, we propose FOCAL, a Role-Separated Attention framework for multi-label node classification on heterogeneous graphs. FOCAL decouples neighborhood aggregation into two complementary roles: a Coverage-Oriented Attention (COA) module that performs flexible, type-aware attention to broadly collect heterogeneous context and highlight rare but informative nodes even from secondary types; and an Anchoring-Oriented Attention (AOA) module that anchors primary semantics via meta-path constrained aggregation over primary neighborhoods, reducing the risk of drifting away from task-critical signals. We further design a role-aware fusion mechanism to coordinate the two views, and adopt training strategies tailored for multi-label and imbalanced supervision, together with a consistency objective to promote complementary yet coherent representations. Our main contributions are:
\begin{itemize}
\item We identify and theoretically characterize two key failure modes on HGN, semantic dilution and coverage constraint, and reveal the meta-path-based anchoring dilemma. We further show that these failure modes are amplified under multi-label supervision.
\item We propose FOCAL, a role-separated attention framework that explicitly decouples semantic coverage (COA) and semantic anchoring (AOA) and integrates them via role-aware fusion for heterogeneous multi-label prediction. We theoretically justify that FOCAL achieves a better coverage-anchoring trade-off than attention-only and meta-path-only alternatives.
\item Extensive experiments on several heterogeneous multi-label benchmarks show that FOCAL significantly outperforms other baselines.
\end{itemize}

\section{Preliminaries}\label{Preliminaries}

\subsection{Problem Definition}
A heterogeneous graph is defined as
\(
\mathcal{G} = (\mathcal{V}, \mathcal{E}, \phi, \psi),
\)
where $\mathcal{V}$ and $\mathcal{E}$ denote the sets of nodes and edges, respectively.
The mapping $\phi: \mathcal{V} \rightarrow \mathcal{A}$ assigns each node to a node type, and
$\psi: \mathcal{E} \rightarrow \mathcal{R}$ assigns each edge to a relation type.
Each node $v_i \in \mathcal{V}$ is associated with a feature vector
$\mathbf{x}_i \in \mathbb{R}^{d_{\phi(v_i)}}$.
Let $\mathcal{V}_l \subset \mathcal{V}$ denote the labeled nodes and
$\mathcal{V}_u = \mathcal{V} \setminus \mathcal{V}_l$ the unlabeled ones.

In the multi-label setting, each labeled node $v_i \in \mathcal{V}_l$ is associated with a subset of labels from a label set $\mathcal{L}$ with $|\mathcal{L}| = C$.
The ground-truth labels are represented by a multi-hot vector
$\mathbf{y}_i \in \{0,1\}^C$.
The goal of heterogeneous multi-label node classification is to predict the labels of unlabeled nodes based on the observed graph structure and node attributes.

\subsection{Generic Heterogeneous GNN Encoder}\label{Generic Heterogeneous GNN Encoder}

We consider a generic heterogeneous graph neural network encoder that maps each node $v_i \in \mathcal{V}$ to a $d$-dimensional representation $h_i \in \mathbb{R}^d$, which is computed by a parameterized function $f_{\theta}$ as: $h_i = f_{\theta}(v_i, \mathcal{G})$, where $\theta$ denotes the learnable parameters.

We focus on attention-based encoders, which cover most existing heterogeneous GNNs. After one aggregation step, the representation of node $v_i$ can be expressed as $h_i = \sum{v_j \in \mathcal{N}(v_i)} \alpha_{j,i} , m_{j \rightarrow i}$, where $\mathcal{N}(v_i)$ denotes the aggregation neighborhood, $\alpha_{j,i}$ are normalized attention weights satisfying $\sum_{v_j \in \mathcal{N}(v_i)} \alpha_{j,i} = 1$, and $m_{j \rightarrow i}$ represent the messages passed from node $v_j$ to node $v_i$. Different heterogeneous GNNs correspond to different choices of aggregation neighborhoods and attention parameterizations.

For analysis, we conceptually distinguish aggregation sources in $\mathcal{N}(v_i)$ according to their relevance to the downstream prediction task.
We refer to a subset of neighbors as \emph{task-critical} (primary) if the information they convey contributes substantially to the prediction of at least one target label, while the remaining neighbors are considered \emph{task-non-critical} (secondary).

\subsection{Representative Attention Paradigms}
Existing attention-based heterogeneous GNNs differ in aggregation strategies, we abstract them into two paradigms: \textbf{Relation-adaptive attention} aggregates information from a heterogeneous neighborhood with learnable attention weights conditioned on node and edge types. \textbf{Meta-path-based attention} restricts aggregation to a predefined set of meta-paths, computing representations along each meta-path and combining them through semantic-level attention.

The two paradigms impose different forms of constraints on attention-based aggregation: relation-adaptive attention relies on parameterized weighting over unrestricted neighborhoods, whereas meta-path-based attention enforces explicit structural restrictions.

\begin{figure*}[htbp]
    \centering
    \includegraphics[width=0.95\linewidth]{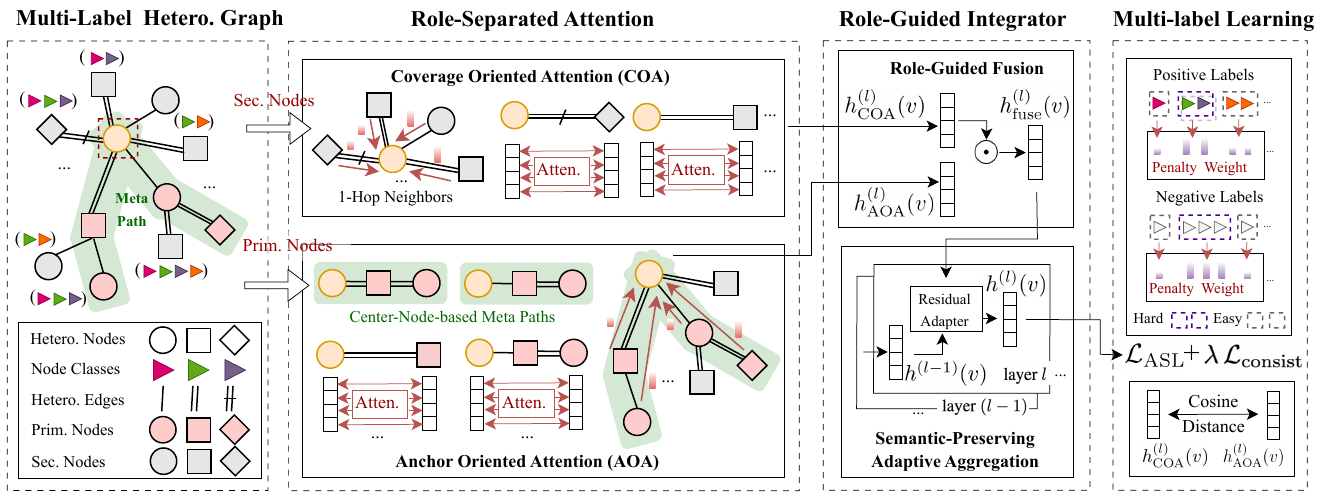}
    \caption{Model overview of FOCAL. We first designs a \textbf{role-separated} attention. 
    In each layer, a coverage-oriented attention (\textbf{COA}) component captures broad heterogeneous contextual semantics from all nodes, 
    while an anchoring-oriented attention (\textbf{AOA}) component models deep primary semantics. Then the \textbf{role-guided integrator} comprises \textbf{role-guided fusion} and \textbf{semantic-preserving adaptive aggregation} to preserve both semantics and prevent primary semantics from being overwritten.
    Finally, an asymmetric loss $\mathcal{L}_{\mathrm{ASL}}$ is adopted for multi-label supervision, along with a consistency loss $\mathcal{L}_{\mathrm{consist}}$ for regularization.}
    \label{fig:rosa_architecture}
    \vspace{-2mm}
\end{figure*}

\section{Theoretical Challenges}\label{sec:theoretical_analysis}
We study two dominant aggregation paradigms for heterogeneous graphs: \textbf{(i) relation-adaptive attention}, and \textbf{(ii) meta-path-based attention}.
For a target node, as introduced in Section~\ref{Generic Heterogeneous GNN Encoder}, we partition its neighbors into a task-critical \emph{\textbf{primary}} set and a task-non-critical \emph{\textbf{secondary}} set, and analyze how the total attention on primary neighbors behaves as heterogeneous neighborhoods expand.

\subsection{Attention Dilution in Relation-adaptive Attention} \label{sec:analysis_relation_adaptive_attention}
We analyze relation-adaptive attention that applies a single softmax over a heterogeneous neighborhood.
For a target node \(t\) with neighbors \(\mathcal{S}(t)\), the attention weights are defined by
\begin{equation}
\label{eq:attn_softmax_def}
\alpha_{t,s}
=
\frac{\exp(E_{t,s})}{\sum_{u\in \mathcal{S}(t)} \exp(E_{t,u})},
\qquad s\in \mathcal{S}(t),
\end{equation}
where \(E_{t,s}\) denotes the attention logit.
We partition neighbors into a task-relevant \emph{primary} set \(\mathcal{S}^*(t)\) and a task-irrelevant \emph{secondary} set \(\mathcal{S}^c(t)\),
with \(n^*:=|\mathcal{S}^*(t)|\) and \(m:=|\mathcal{S}^c(t)|\).
We further define the \emph{total primary attention mass}
\begin{equation}
\label{eq:primary_mass_def1}
A^*(t)
:=
\sum_{s\in \mathcal{S}^*(t)} \alpha_{t,s}.
\end{equation}

\begin{theorem}[Attention dilution of heterogeneous graph]
\label{lem:softmax_attention_dilution}
Assume the within-group empirical averages of \(\exp(E_{t,s})\) converge almost surely to constants
\(\mu^*>0\) (primary) and \(\mu>0\) (secondary), with no \(m\)-dependent exponential advantage.
Then as \(m\to\infty\),
\begin{equation}
\label{eq:attn_dilution_limit_main}
A^*(t)\xrightarrow[]{\mathrm{a.s.}}
\frac{n^*\mu^*}{n^*\mu^*+m\mu}.
\end{equation}
In particular, when \(n^*=O(1)\), if \(\frac{m}{n^*}\to\infty\) then \(A^*(t)\to 0\) and \(A^*(t)=O(1/m)\).
\end{theorem}
See Appendix~\ref{sec:appendix_attn_dilution} for full statements and proofs.

Let the attention aggregator be
\(
\label{eq:attn_agg_def}
h_t=\sum_{s\in\mathcal{S}(t)} \alpha_{t,s} m_{s\to t},
\)
and denote the primary component by \(h_t^{(p)}:=\sum_{s\in\mathcal{S}^*(t)} \alpha_{t,s} m_{s\to t}\).
Consider a multi-label objective with positive label set \(\mathcal{P}(t)\) of size \(L:=|\mathcal{P}(t)|\),
logits \(z_k=w_k^\top h_t\), and sigmoid \(\sigma(\cdot)\).

\begin{theorem}[Multi-label amplification of attention dilution]
\label{lem:multilabel_loss_grad_amp}
Assume \(\|m_{s\to t}\|\le M\) for all \(s\in\mathcal{S}^*(t)\), and for each \(k\in\mathcal{P}(t)\) the logit satisfies
\(z_k \le a\|h_t^{(p)}\|+b\) for some constants \(a>0\), \(b\in\mathbb{R}\).
Then the positive-label BCE loss obeys
\begin{equation}
\label{eq:multilabel_amp_main_loss}
\sum_{k\in\mathcal{P}(t)} -\log \sigma(z_k)
\;\ge\;
L\cdot\Big(-\log \sigma\big(aM\,A^*(t)+b\big)\Big).
\end{equation}
Moreover, if \(\|\nabla_{m_{s\to t}} z_k\|\le C\,\alpha_{t,s}\) holds for all \(k\in\mathcal{P}(t)\) and \(s\in\mathcal{S}^*(t)\),
then the gradient through primary messages is bounded by
\begin{equation}
\label{eq:multilabel_amp_main_grad}
\Big\|\nabla_{\{m_{s\to t}:s\in\mathcal{S}^*(t)\}}
\sum_{k\in\mathcal{P}(t)} -\log \sigma(z_k)\Big\|
=
O\big(L\cdot A^*(t)\big).
\end{equation}
\end{theorem}
See Appendix~\ref{app:multilabel_amp_full} for full statements and proofs.

These results indicate a fundamental limitation of relation-adaptive attention.
A single softmax over an expanding heterogeneous neighborhood inevitably dilutes attention: as the number of secondary neighbors grows, the total primary attention mass \(A^*(t)\) decays (typically \(A^*(t)=O(n^*/m)\) and \(A^*(t)\to 0\) when \(m/n^*\to\infty\); Theorem~\ref{lem:softmax_attention_dilution}). In multi-label learning, this dilution is amplified: both the positive-label loss lower bound scales linearly with the number of positives \(L\) (Theorem~\ref{lem:multilabel_loss_grad_amp}) and, more critically, the gradient signal along primary messages is suppressed as \(O(LA^*(t))\) (Proposition~\ref{prop:gradient_attenuation}).

\subsection{Coverage-Dilution Dilemma in Meta-path-based Attention}
\label{subsec:meta_path_attention_analysis}

Given a predefined meta-path set \(\mathcal{M}\) and meta-path \(\mathcal{P} \in \mathcal{M}\), the representation of node \(v\) is computed as
\(
\label{eq:han_representation}
h_v
=
\sum_{\mathcal{P}\in\mathcal{M}}
\beta_{\mathcal{P}}(v)\, z_v^{\mathcal{P}},
\)
where \(\sum_{\mathcal{P}\in\mathcal{M}}\beta_{\mathcal{P}}(v)=1\), \(\beta_{\mathcal{P}}(v)\ge 0\), and \(z_v^{\mathcal{P}}\) is the meta-path-specific aggregation along \(\mathcal{P}\).
All node representations therefore lie in the meta-path-induced subspace
\(
\label{eq:meta_path_span_setup}
\mathcal{H}_{\mathcal{M}}
:=
\mathrm{span}\left(\{ z_v^{\mathcal{P}} \mid \mathcal{P}\in\mathcal{M} \}\right).
\)
Let \(\mathcal{M}^*\subset\mathcal{M}\) denote the primary meta-path subset and define the total primary semantic attention mass
\begin{equation}
\label{eq:primary_mass_def2}
B^*(v)
:=
\sum_{\mathcal{P}\in\mathcal{M}^*}\beta_{\mathcal{P}}(v).
\end{equation}

\begin{theorem}[The coverage-dilution dilemma of meta-path-based attention]
\label{lem:meta_path_subspace_dilution_main}
Meta-path attention with a discrete set \(\mathcal{M}\) is constrained in two coupled ways:
\textbf{(i) Coverage constraint.}
For all nodes \(v\), \(h_v\in \mathcal{H}_{\mathcal{M}}\). Hence, information not captured through \(\{z_v^{\mathcal{P}}\}_{\mathcal{P}\in\mathcal{M}}\) cannot be expressed by \(h_v\).
\textbf{(ii) Semantic dilution.}
If \(\beta_{\mathcal{P}}(v)\) is produced by a single semantic-level softmax over \(\mathcal{M}\), then under mild non-degeneracy conditions on semantic scores (Appendix~\ref{app:meta_path_tradeoff_full}), \( B^*(v) = O\left(|\mathcal{M}^*| / |\mathcal{M}|\right) \).
In particular, when \(|\mathcal{M}|\) increases while \(|\mathcal{M}^*|\) is fixed, the total attention allocated to primary semantics decreases.
\end{theorem}

See Appendix~\ref{app:meta_path_tradeoff_full} for full statements and proofs.

We abstract semantic dilution as attenuation of the primary component:
\(
\label{eq:attenuation_model}
h_v = B^*(v)\, h_v^* + r_v,
\)
where \(h_v^*\) is the primary semantic component and \(r_v\) aggregates secondary components.

Let \(\mathcal{P}(v)\) be the set of positive labels for node \(v\), with \(|\mathcal{P}(v)|=L\).
A standard multi-label head uses
\(
\label{eq:multilabel_head}
z_{v,k} = w_k^\top h_v,\qquad k=1,\dots,K,
\)
and optimizes the BCE loss
\(
\label{eq:bce_loss}
\mathcal{L}_{\mathrm{ML}}(v) = \sum_{k=1}^K \Big(
- y_{v,k}\log \sigma(z_{v,k})
- (1-y_{v,k})\log\big(1-\sigma(z_{v,k})\big)
\Big),
\)
where
\(
y_{v,k}\in\{0,1\}.
\)

\begin{theorem}[Multi-label amplification of coverage-dilution dilemma]
\label{lem:multilabel_amplification_main}
Assume \(\|w_k\|\le W\) for all \(k\) and \(\|h_v^*\|\le H\) and setting \(r_v=0\) for
\(h_v = B^*(v)\, h_v^* + r_v\), as the number of positive labels \(\mathcal{P}(v)\) grows:

\textbf{(i) Positive-loss floor grows.}
\begin{equation}
\label{eq:pos_loss_lowerbound_main}
\sum_{k\in \mathcal{P}(v)} \log\bigl(1+\exp(-z_{v,k})\bigr)
\;\ge\;
|\mathcal{P}(v)|\,\mathcal{C}
-\frac{|\mathcal{P}(v)|}{2}WHB^{*}(v)
\end{equation}
Where $\mathcal{C}$ is constant and  as \(B^*(v)\to 0\), the bound approaches \(|\mathcal{P}(v)|\ \mathcal{C}\) up to \(O(B^*(v))\).

\textbf{(ii) Positive-label error terms accumulate.}
\begin{equation}
\label{eq:grad_bound_main}
\Big|\nabla_{h_v}\mathcal{L}_{\mathrm{ML}}(v)\Big|
\;\le\;
W \sum_{k=1}^K \big|\sigma(z_{v,k})-y_{v,k}\big|.
\end{equation}
Moreover, for any \(\delta\ge 0\), if \(|z_{v,k}|\le \delta\) and \(y_{v,k}=1\), then
\begin{equation}
\label{eq:pos_error_lower_simple_main}
\big|\sigma(z_{v,k})-1\big|
\;\ge\;
\sigma(-\delta).
\end{equation}
Hence, if \(|z_{v,k}|\le \delta\) holds for all \(k\in\mathcal{P}(v)\), then
\begin{equation}
\label{eq:sum_pos_error_lower_main}
\sum_{k=1}^K \big|\sigma(z_{v,k})-y_{v,k}\big|
\;\ge\;
|\mathcal{P}(v)|\,\sigma(-\delta),
\end{equation}
so the cumulative positive-label contribution is \(\Omega(|\mathcal{P}(v)|)\) for fixed small \(\delta\).
\end{theorem}

See Appendix~\ref{app:multilabel_attenuation_full} for full statements and proofs.

These results reveal an inherent dilemma of meta-path-based attention: limiting the meta-path set will cause coverage constraint while enlarging it can simultaneously dilute primary semantics.
Theorem~\ref{lem:meta_path_subspace_dilution_main} explains how a discrete meta-path prior can both (i) constrain representations to \(\mathcal{H}_{\mathcal{M}}\) and (ii) dilute primary semantics (\(B^*(v)\downarrow\)) as \(|\mathcal{M}|\) grows.
Theorem~\ref{lem:multilabel_amplification_main} shows that once \(B^*(v)\) is small, multi-label learning suffers a per-positive loss floor near \(\log 2\) and an \(O(|\mathcal{P}(v)|)\) accumulation of positive-label error terms (and thus potentially large gradient pressure near zero logits).
These effects arise from representational attenuation in \(h_v\), rather than the specific output normalization.

\subsection{Discussion}
Across both paradigms, our analysis identifies a common bottleneck: primary semantics are progressively attenuated in heterogeneous attention mechanisms, and this attenuation is systematically amplified under multi-label supervision.

Specifically, relation adaptive attention mechanisms, while flexible,
lack explicit structural anchoring and are prone to semantic dilution
as heterogeneous neighborhoods expand.
In contrast, meta-path-based methods impose strong structural inductive biases,
but face a dilemma between semantic coverage and semantic dilution.
Both of them is further amplified in multi-label settings.

Taken together, an effective solution must simultaneously preserve hard structural focus
on primary semantics and while maintaining calibrated capacity for secondary semantics,
motivating the role-separate model design proposed next.

\section{Model: FOCAL}\label{Methodology}
To address above challenges, we propose \textbf{RoSA} (\textbf{Ro}le-\textbf{S}eparated \textbf{A}ttention for Multi-label Heterogeneous Graphs).
The overall framework of FOCAL is shown in Figure.~\ref{fig:rosa_architecture}.

\subsection{Role-Separated Attention Mechanism}\label{sec:rosa_arch}
FOCAL adopts a role-separated attention architecture composed of two complementary attention components. At each layer, node representations are updated in parallel by a coverage-oriented attention (COA) module and an anchoring-oriented attention (AOA) module. The outputs of these two components are subsequently integrated to form the updated node representations, which are then propagated to the next layer. 

\paragraph{Coverage-Oriented Attention (COA)}
This component aims to maximize semantic coverage in heterogeneous graphs by aggregating information from diverse node and relation types.
Unlike anchoring-oriented mechanisms, COA does not impose structural constraints on the aggregation space.
Instead, it treats all heterogeneous neighbors as potential semantic contributors.
This enables flexible modeling of relational and attribute interactions across the full heterogeneous neighborhood and also improves the visibility of diverse secondary semantics.

Formally, we denote \(\mathcal{A}_{\mathrm{cov}}^{(l)}\) as the COA module at layer \(l\).
For a target node \(t\), COA operates on the unconstrained heterogeneous neighborhood 
\(
\mathcal{C}(t) := \bigcup_{r \in \mathcal{R}} \mathcal{N}_r(t),
\)
where \(\mathcal{R}\) denotes the set of relation types and \(\mathcal{N}_r(t)\) represents the neighbors of \(t\) connected via relation \(r\).
Given \(\mathcal{C}(t)\), COA updates the node representation by a learnable, relation-aware attention operator:
\begin{equation}
h^{(l)}_{\mathrm{COA}}(t)
=
\mathcal{A}_{\mathrm{cov}}^{(l)}
\left(
\mathcal{C}(t),\,
h^{(l-1)}(s)
\right),
\end{equation}
where \(\mathcal{A}_{\mathrm{cov}}\) is parameterized to flexibly model heterogeneous interactions, prioritizing expressive aggregation over semantic concentration.

We instantiate COA using a transformer-style heterogeneous attention mechanism.
For head \(i\in\{1,\dots,h\}\), we compute
\begin{equation}
\alpha^{i,cov}_{t,s}
=
\operatorname*{softmax}\limits_{s\in\mathcal{C}(t)}
\left(
\frac{
\left(Q^{i}(t)\right)^\top K^{i}(s)
}{
\sqrt{d}
}
\cdot \mu^{(l)}_{r(t,s),i}
\right),
\end{equation}
where \(r(t,s)\) denotes the relation type between \(t\) and \(s\),
\(Q^i(t)=W^i_Q h^{(l-1)}(t)\), \(K^i(s)=W^i_K h^{(l-1)}(s)\), and \(V^i(s)=W^i_V h^{(l-1)}(s)\).
The COA output is
\begin{equation}
h^{(l)}_{\mathrm{COA}}(t)
=
\big\|_{i=1}^{d}
\sum_{s\in \mathcal{C}(t)}
\alpha^{i,cov}_{t,s}\,V^{i}(s).
\end{equation}
To adapt aggregation to heterogeneous relations, COA introduces a learnable relation-head reweighting term \(\mu^{(l)}_{r,i} > 0\),
which modulates the attention strength of each relation type without imposing structural constraints.

\begin{table*}[htbp]
\centering
\caption{Overall performance in  multi-label node classification task for heterogeneous graphs. Best results are in \textbf{bold}, and second-best are \underline{underlined}.}
\label{tab:main_results}
\resizebox{\linewidth}{!}{
\begin{tabular}{l ccc ccc ccc}
\toprule
 \multirow{2}{*}{\textbf{Model}} &
\multicolumn{3}{c}{\textbf{IMDB}} &
\multicolumn{3}{c}{\textbf{Amazon}} &
\multicolumn{3}{c}{\textbf{CITE}} \\
\cmidrule(lr){2-4}\cmidrule(lr){5-7}\cmidrule(lr){8-10}
& \textbf{Micro-F1} $\uparrow$& \textbf{Macro-F1} $\uparrow$ & \textbf{Sample-F1} $\uparrow$
& \textbf{Micro-F1} $\uparrow$& \textbf{Macro-F1} $\uparrow$ & \textbf{Sample-F1} $\uparrow$
& \textbf{Micro-F1} $\uparrow$ & \textbf{Macro-F1} $\uparrow$ & \textbf{Sample-F1} $\uparrow$\\
\midrule

RGAT & $\underline{0.6286}_{\pm0.0041}$ & $0.5831 _{\pm0.0035}$ & $0.6113 _{\pm0.0039}$ & $0.7616 _{\pm0.0058}$ & $0.5015 _{\pm0.0072}$ & $0.7518 _{\pm0.0061}$ & $0.2925 _{\pm0.0044}$ & $0.0055 _{\pm0.0002}$ & $0.2490 _{\pm0.0089}$ \\
MAGNN & $0.6046 _{\pm0.0045}$ & $0.5601 _{\pm0.0038}$ & $0.5776 _{\pm0.0042}$ & $\underline{0.9010} _{\pm0.0031}$ & $0.5190 _{\pm0.0069}$ & $\underline{0.8541} _{\pm0.0040}$ & $0.3851_{\pm0.0051}$ & $0.0064 _{\pm0.0003}$ & $\underline{0.3820} _{\pm0.0095}$ \\
SimpleHGN& $0.6274_{\pm0.0052}$ & $0.5898 _{\pm0.0047}$ & $0.6093 _{\pm0.0050}$ & $0.8606 _{\pm0.0043}$ & $\underline{0.6564}_{\pm0.0055}$ & $0.8502 _{\pm0.0045}$ & $0.3623 _{\pm0.0048}$ & $0.0062 _{\pm0.0003}$ & $0.3463 _{\pm0.0101}$ \\
HPN& $0.6003 _{\pm0.0061}$ & $0.5537 _{\pm0.0053}$ & $0.5744 _{\pm0.0058}$ & $0.5683 _{\pm0.0088}$ & $0.1272 _{\pm0.0105}$ & $0.5765 _{\pm0.0092}$ & $0.3341 _{\pm0.0067}$ & $0.0060 _{\pm0.0004}$ & $0.3302 _{\pm0.0112}$ \\
\midrule
CorGCN & $0.5752 _{\pm0.0082}$ & $0.1228 _{\pm0.0115}$ & $0.5171 _{\pm0.0091}$ & $0.5991 _{\pm0.0103}$ & $0.5487 _{\pm0.0098}$ & $0.5774 _{\pm0.0108}$ & $\underline{0.3736} _{\pm0.0055}$ & $0.0063 _{\pm0.0004}$ & $0.3706 _{\pm0.0093}$ \\
TriPer & $0.6628 _{\pm0.0451}$ & $\underline{0.6462} _{\pm0.0512}$ & $0.5310 _{\pm0.0398}$ & $0.5660 _{\pm0.0423}$ & $0.3075 _{\pm0.0288}$ & $0.5162 _{\pm0.0415}$ & $0.1124 _{\pm0.0156}$ & $0.0070 _{\pm0.0015}$ & $0.0538 _{\pm0.0112}$ \\
\midrule
FOCAL-COA & $0.6277 _{\pm0.0033}$ & $0.6002_{\pm0.0029}$ & $\underline{0.6196} _{\pm0.0031}$ & $0.7480 _{\pm0.0048}$ & $0.6285 _{\pm0.0051}$ & $0.7797 _{\pm0.0045}$ & $0.2567 _{\pm0.0039}$ & $\textbf{0.0255} _{\pm0.0011}$ & $0.2484 _{\pm0.0078}$ \\
FOCAL-AOA & $0.5708 _{\pm0.0075}$ & $0.5047 _{\pm0.0068}$ & $0.5327 _{\pm0.0071}$ & $0.6290 _{\pm0.0081}$ & $0.1904 _{\pm0.0120}$ & $0.6256 _{\pm0.0085}$ & $0.3729 _{\pm0.0053}$ & $0.0063 _{\pm0.0004}$ & $0.3712 _{\pm0.0099}$ \\
\midrule
\textbf{FOCAL} & $\textbf{0.6692} _{\pm0.0021}$ & $\textbf{0.6477} _{\pm0.0018}$ & $\textbf{0.6513} _{\pm0.0020}$ & $\textbf{0.9305} _{\pm0.0015}$ & $\textbf{0.8184} _{\pm0.0025}$ & $\textbf{0.9249} _{\pm0.0017}$ & $\textbf{0.4224} _{\pm0.0032}$ & $\underline{0.0153} _{\pm0.0009}$ & $\textbf{0.4115} _{\pm0.0056}$ \\
\bottomrule
\end{tabular}
}
\end{table*}

\paragraph{Anchoring-Oriented Attention (AOA)}
This component is designed to preserve and stabilize primary semantics by restricting attention to a structurally anchored neighborhood. 
Unlike COA, which aggregates from an unconstrained heterogeneous neighborhood to maximize semantic coverage, AOA intentionally narrows the aggregation space to suppress interference from secondary nodes and mitigate semantic dilution.
Similarly, we denote \(\mathcal{A}_{\mathrm{anc}}^{(l)}\) as the AOA module at layer \(l\).
In our setting, we manually specify meta-path $\mathcal{P}$ to characterize the primary structural pattern. 
The meta-path $\mathcal{P}$ is selected based on domain knowledge to reflect the task-critical primary semantics.
For a target node $t$, let $\mathcal{N}_{\mathcal{P}}(t)$ denote the set of nodes reachable from $t$ via $\mathcal{P}$. 
AOA performs attention-based aggregation only over $\mathcal{N}_{\mathcal{P}}(t)$, yielding
\begin{equation}
h^{(l)}_{\mathrm{AOA}}(t)
=
\mathcal{A}_{\mathrm{anc}}^{(l)}
\left(
\mathcal{N}_{\mathcal{P}}(t),\,
h^{(l-1)}(s)
\right).
\end{equation}
Since attention is normalized within $\mathcal{N}_{\mathcal{P}}(t)$, AOA concentrates representational capacity on the anchored structural semantics and reduces the dilution effect caused by heterogeneous noise. We instantiate $\mathcal{A}_{\mathrm{anc}}$ using a multi-head attention operator defined on the meta-path-induced neighborhood:
\begin{equation}
h^{(l)}_{\mathrm{AOA}}(t)
=
\big\|_{i=1}^{d}
\sum_{s\in\mathcal{N}_{\mathcal{P}}(t)}
\alpha^{i,anc}_{t,s}\;
W^{i}_{\mathcal{P}}\, h^{(l-1)}(s),
\end{equation}
where $\|\,$ denotes concatenation over heads and $W^{i}_{\mathcal{P}}$ is the learnable projection matrix for head $i$ under meta-path $\mathcal{P}$.
The attention coefficient 
\(
\alpha^{i,anc}_{t,s}
=
\operatorname{softmax}_{s\in\mathcal{N}_{\mathcal{P}}(t)}\left(e^{i}_{t,s}\right),
\)
where \(e^{i}_{t,s}\)
is computed as
\begin{equation}
\begin{aligned}
e^{i}_{t,s}
&=
\mathrm{LeakyReLU}\left(
\left(a^{i}_{\mathcal{P}}\right)^{\top}
\left[
W^{i}_{\mathcal{P}} h^{(l-1)}(t)
\,\|\,
W^{i}_{\mathcal{P}} h^{(l-1)}(s)
\right]
\right)
\end{aligned}
\end{equation}
By anchoring aggregation to a single, manually specified meta-path neighborhood (for extension to multiple, see Appendix~\ref{app:aoa_multi_metapath}), AOA provides a stable structural backbone that prioritizes primary semantics, while deliberately sacrificing heterogeneous coverage. 
This anchored representation complements the coverage-oriented representation produced by COA and stabilizes learning under multi-label supervision.

\subsection{Role-Guided Adaptive Integrator}
\paragraph{Role-Guided Fusion Mechanism.}
At each layer, FOCAL produces two complementary representations for a target node $t$: a coverage-oriented representation  and an anchoring-oriented representation. Since these two components serve fundamentally different roles, the integration mechanism explicitly preserves the anchoring function of AOA while allowing COA to act as a contextual augmentation. To this end, FOCAL adopts a role-aware integration mechanism composed of a bidirectional gated fusion module and a semantic-preserving residual propagation scheme.
\begin{equation}
\begin{aligned}
g^{(i)}_t &= \sigma\left(W^{(i)}_g \left[ h^{(l)}_{\mathrm{COA}}(t) \,\|\, h^{(l)}_{\mathrm{AOA}}(t) \right] + b^{(i)}_g \right),
\end{aligned}
\end{equation}
where $i\in\{1,2\}$, $\sigma$ denotes the sigmoid function, $\,\|\,$ represents concatenation, and $g^{(1)}_t, g^{(2)}_t \in (0,1)^d$ are dimension-wise gating vectors. The fused representation is then computed as:
\begin{equation}
h^{(l)}_{\mathrm{fuse}}(t)
= g^{(1)}_t \odot h^{(l)}_{\mathrm{COA}}(t)
+ g^{(2)}_t \odot h^{(l)}_{\mathrm{AOA}}(t),
\end{equation}
where $\odot$ denotes element-wise multiplication.
The two gates are independently parameterized and are not constrained to sum to one. This design enables non-symmetric integration: the model can amplify, suppress, or selectively preserve signals from each channel. As a result, AOA is adaptively protected against being overwhelmed by heterogeneous contextual signals from COA.
\paragraph{Semantic-Preserving Adaptive Aggregation.}
To prevent primary semantics from being overwritten during deep propagation, FOCAL further introduces a node-adaptive residual connection.
Given the fused representation $h^{(l)}_{\mathrm{fuse}}(t)$ and the previous-layer representation $h^{(l-1)}(t)$, we define:
\begin{equation}
h^{(l)}(t)
= \alpha^{(l)}_t \odot h^{(l)}_{\mathrm{fuse}}(t)
+ \left(1 - \alpha^{(l)}_t \right) \odot h^{(l-1)}(t),
\end{equation}
where $ \alpha^{(l)}_t = \sigma\left(W_{\alpha} \left[ h^{(l)}_{\mathrm{fuse}}(t) \,\|\, h^{(l-1)}(t) \right] + b_{\alpha}
\right)$.
This residual formulation functions as a semantic-preserving channel, ensuring that primary semantics stabilized by AOA can be consistently propagated across layers.

Through this role-aware integration mechanism, FOCAL couples heterogeneous semantic coverage with structurally anchored representations, while maintaining a stable semantic backbone across layers. This design directly addresses the semantic attenuation and multi-label amplification effects identified in Section~\ref{sec:theoretical_analysis}.

\subsection{Joint Multi-label Learning}
Multi-label learning on heterogeneous graphs is typically characterized by label imbalance~\cite{shi2020multi}, where frequent and rare labels coexist and contribute unequally to the training dynamics. To account for this, FOCAL adopts the Asymmetric Loss (ASL)~\cite{ridnik2021asymmetric} as a part of task objective, defined as $\mathcal{L}_{\mathrm{ASL}}$
which modulates the contribution of samples according to label occurrence frequency and prediction difficulty.

FOCAL further introduces a consistency regularizer to coordinate the two role-separated representations. As discussed in Section~\ref{sec:rosa_arch}, COA and AOA respectively model secondary heterogeneous semantics and primary anchored semantics. To this end, we impose a cosine-based consistency constraint between the two components:
\begin{equation}
\mathcal{L}_{\mathrm{consist}} =
\frac{1}{|\mathcal{V}|}
\sum_{t \in \mathcal{V}}
\left(
1 -
\frac{
h^{(L)}_{\mathrm{COA}}(t)^{\top} h^{(L)}_{\mathrm{AOA}}(t)
}{
\big|h^{(L)}_{\mathrm{COA}}(t)\big| \, \big|h^{(L)}_{\mathrm{AOA}}(t)\big|
}
\right).
\end{equation}
This term encourages the coverage-oriented and anchoring-oriented representations to remain geometrically aligned while retaining their role-specific characteristics. The final  objective of FOCAL is defined as: 
\begin{equation}
\mathcal{L}
= \mathcal{L}_{\mathrm{ASL}} + \lambda \, \mathcal{L}_{\mathrm{consist}}.
\end{equation}
where $\lambda$ controls the strength of the consistency constraint.

\begin{theorem}\label{thm:rosa_main}
FOCAL addresses the attention dilution and coverage-dilution dilemma problems by lifting the lower bound of attention dilution to $\gamma_v := \min_{k} g^{(2)}_{v,k}$ by AOA, while COA preserves coverage over secondary nodes.
\end{theorem}
See Appendix~\ref{app:theory} for full statements and proofs.

\section{Experiments}
\subsection{Experimental Setup}
\paragraph{Datasets} We evaluate FOCAL on three heterogeneous graph datasets: IMDB~\cite{lv2021we}, Amazon~\cite{gupta2025persona}, and CITE~\cite{zhang2025cite}. 
All three are heterogeneous graph datasets designed for multi-label node classification, spanning different domains and graph scales.
Detailed statistics and construction procedures are provided in Appendix~\ref{appendix:exp_setting}.

\begin{table}[htbp]
\centering
\caption{Statistal information of datasets used in this paper.}
\resizebox{\linewidth}{!}{
\begin{tabular}{cccccccc}
\toprule
\textbf{Dataset} & \textbf{Domain} & \textbf{Node Type} & \textbf{\# Nodes} & \textbf{Edge Type} &\textbf{\# Edges} &\textbf{\# Label} \\
\midrule
\textbf{IMDB} & Movie & \makecell{Movie, Actor, \\Director, \\Keyword} & 21,420 & \makecell{Movie-Actor, \\Movie-Director, \\Movie-Keyword} &40,341 & 5  \\
\midrule
\textbf{Amazon} & E-commerce & \makecell{User, Product, \\Persona} & 13,973& \makecell{User-Product, \\Product-Persona}  &17,251 & 6   \\
\midrule 
\textbf{CITE} & Chemistry & \makecell{Paper, \\Keyword, \\Substrate, \\Author} & 213,277 & \makecell{Paper-Keyword, \\Paper-Paper, \\Paper-Substrate,\\ Paper-Author}  & 938,473 & 107   \\
\bottomrule
\end{tabular}
}
\vspace{-2mm}
\label{tab:dataset_info}
\end{table}

\paragraph{Baselines}
We compare FOCAL with a suite of representative baselines covering different modeling paradigms. 
(1) \textbf{Heterogeneous GNNs:} RGAT~\cite{busbridge2019relational}, MAGNN~\cite{fu2020magnn}, SimpleHGN~\cite{lv2021we}, HPN~\cite{ji2021heterogeneous}, including attention-based heterogeneous aggregation and meta-path-based aggregation paradigms.
(2) \textbf{Multi-label Heterogeneous Graph Models:} CorGCN~\cite{bei2025correlation} and TriPer~\cite{gupta2025persona}. 
Single-label baselines are adapted by replacing \texttt{softmax} with \texttt{sigmoid} and setting the threshold to 0.5.

\paragraph{Implementation Details} For each dataset, we distinguish primary and secondary semantic sources according to the dataset schema and task definition. Specifically, persona, keywords, and substrate nodes are treated as primary semantic anchors for Amazon, IMDB, and CITE respectively, while the remaining node types are regarded as secondary semantic sources. We report mean\(\pm\)std over 5 runs. Model configurations and their hyper-parameters are detailed in~\ref{appendix:exp_setting}. 
We take Micro-F1, Macro-F1, and Sample-F1 as evaluation metrics.

\subsection{Overall Performance}
Table~\ref{tab:main_results} reports the results of heterogeneous multi-label node classification. FOCAL achieves the best results against all baselines on all 9 metrics. Compared to the strongest baseline, FOCAL improves Micro-F1 by 0.64-3.73 points and Sample-F1 by 2.95-7.08 points across datasets.
We attribute the gains to FOCAL’s role-separated aggregation: AOA anchors message passing on task-aligned primary semantics, whereas COA maintains coverage over the remaining node types and complementary relations; the two signals are then adaptively fused.
On IMDB, where keywords provide direct cues for movie categories, FOCAL’s anchoring helps prevent attention from being diluted by auxiliary relations (e.g., actor/director), leading to the strongest and most balanced performance across metrics. On Amazon, persona serves as a highly informative anchor for predicting user attributes; FOCAL further benefits from covering user-product interactions, yielding consistent gains over strong attention-based heterogeneous aggregators. On the most challenging CITE benchmark, FOCAL remains clearly superior, demonstrating scalability under large-graph, long-tailed conditions. The results confirm that separating primary anchoring from secondary coverage is crucial for stable and transferable multi-label learning on heterogeneous graphs.
Overall, these results are consistent with the coverage-dilution dilemma analysis in Section~\ref{sec:theoretical_analysis} and role-separated design of Section~\ref{Methodology}. 

\subsection{Over-Smoothing Effects}
We study the over-smoothing behavior under deep stacking by varying the number of message-passing layers \(L\) while keeping the remaining hyperparameters fixed. The results are shown in Figure.~\ref{fig:exp_smoothing}
As depth increases, other baselines exhibit a clear Micro-F1 drop, consistent with over-smoothing: repeated propagation progressively homogenizes node representations and amplifies irrelevant heterogeneous signals, reducing class separability. In contrast, FOCAL remains stable and even slightly improves with more layers, indicating better preservation of discriminative semantics under deep stacking. We attribute this robustness to \textbf{(1)} a bidirectional adaptive fusion gate that prevents the anchoring channel from being overwhelmed by contextual noise, and \textbf{(2)} an adaptive residual propagation that retains a semantic backbone across layers. Overall, the depth-sensitivity gap highlights FOCAL’s stronger resistance to over-smoothing.

\begin{figure}[htbp]
    \centering
    \includegraphics[width=\linewidth]{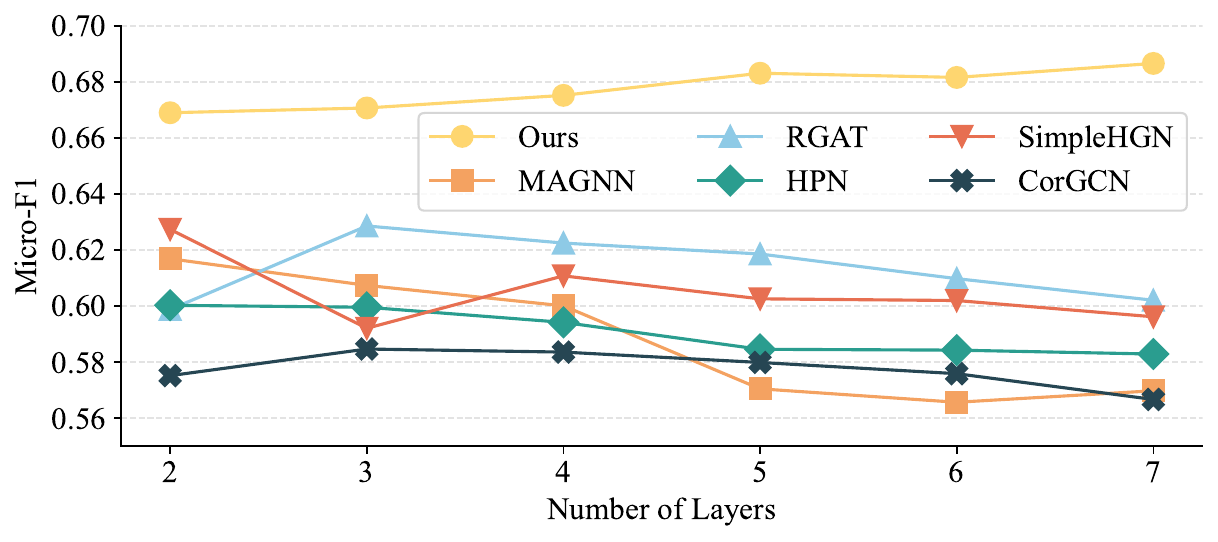}
    \caption{Results of over-smoothing effects.}
    \label{fig:exp_smoothing}
    \vspace{-4mm}
\end{figure}

\subsection{Model Analysis}

\paragraph{Running Efficiency}
We evaluate the running efficiency of FOCAL and other baselines. Figure~\ref{fig:run_time} reports the average running time per training epoch and the Micro-F1 results.
FOCAL achieves a balance, delivering strong Micro-F1 with only moderate per-epoch overhead compared to lightweight baselines, which are faster but notably less accurate. Moreover, FOCAL is more effective than heavier competitors, requiring less training time while attaining better accuracy. Overall, these results suggest that FOCAL lies close to the Pareto frontier of accuracy-efficiency trade-offs under the same setting.

\begin{figure}[htbp]
    \centering
    \includegraphics[width=\linewidth]{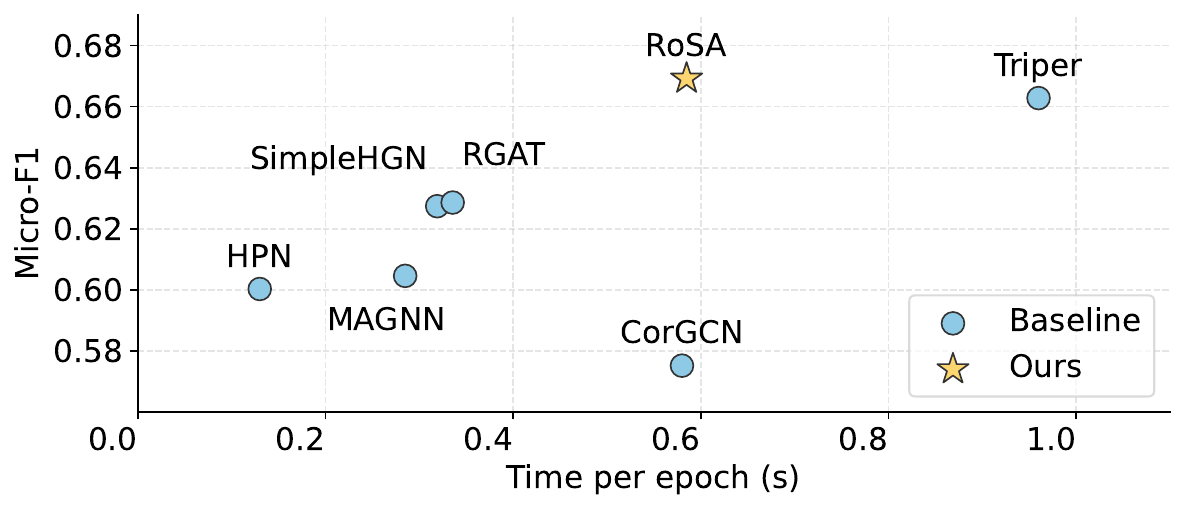}
    \vspace{-6mm}
    \caption{The average running time for each testing sample.}
    \label{fig:run_time}
    \vspace{-4mm}
\end{figure}

\paragraph{Ablation Study}
To assess the contribution of key components in FOCAL, we conduct ablation studies (results in the appendix). We consider variants that \textbf{(1)} remove COA or AOA, \textbf{(2)} replace the bidirectional gated fusion with simplified fusion (single-gate or naive sum/concat) and replace the adaptive residual with a fixed residual, and \textbf{(3)} drop ASL or the branch consistency regularizer. The results are provided in Table~\ref{tab:main_results} and Table~\ref{tab:ablation}. Overall, removing any component consistently degrades performance, confirming that FOCAL benefits from both role separation and controlled integration. See detailed analysis in Appendix~\ref{Ablation_Study}.

\begin{table}[htbp]
\centering
\caption{Ablation study on the IMDB dataset.}
\label{tab:ablation}
\setlength{\tabcolsep}{6pt}
\resizebox{\linewidth}{!}{
\begin{tabular}{lccc}
\toprule
\textbf{Model} & \textbf{Micro-F1} & \textbf{Macro-F1} & \textbf{Sample-F1} \\

\midrule
FOCAL-w/o gate fusion & 0.6586 & 0.6297 & 0.6231\\
FOCAL-w/o sum/concat fusion & 0.6399 & 0.6121 & 0.6443 \\
FOCAL-w/o adaptive residual & 0.6631 & 0.6375 & 0.6495 \\
\midrule
 FOCAL-w/o $\mathcal{L}_\text{ASL}$ & 0.6362 & 0.6030 & 0.6154 \\
 FOCAL-w/o $\mathcal{L}_\text{consist}$ & 0.6638 & 0.6391 & 0.6497 \\
\midrule
\textbf{FOCAL} & \textbf{0.6692} & \textbf{0.6477} & \textbf{0.6513} \\
\bottomrule
\end{tabular}
}
\end{table}

\paragraph{Parameter Analysis}
We conduct sensitivity analyses on the IMDB dataset to investigate the robustness of FOCAL with respect to key hyperparameters like the consistency weight $\lambda$, the choice of primary node type and the hidden dimension $d$. The results are shown in Figure.~\ref{fig:exp_param}. Using keywords as anchors (MKM) performs best, while anchoring on actors (MAM) / directors (MDM) or treating all node types (All) as primary degrades performance, indicating that effective anchoring should focus on the most label-discriminative semantics. With MKM fixed, performance peaks at a small $\lambda$ (around [0.05, 0.10]) and drops as $\lambda$ increases, suggesting that consistency regularization is helpful but should remain moderate to avoid collapsing the complementary roles of COA and AOA. 
We further vary the hidden dimension $d$ from 16 to 256. Increasing $d$ consistently improves all performances with the largest improvements observed when moving from small to moderate, followed by diminishing returns at larger value. This trend indicates that FOCAL benefits from sufficient capacity to encode heterogeneous semantics, while remaining stable to the choice of $d$ once a moderate dimensionality is reached. Detailed analyses are shown in Appendix~\ref{Parameter_Analysis}.

\begin{figure}[htbp]%
    \centering 
     \subfloat[$\lambda$]
     { \includegraphics[width=0.31\linewidth]{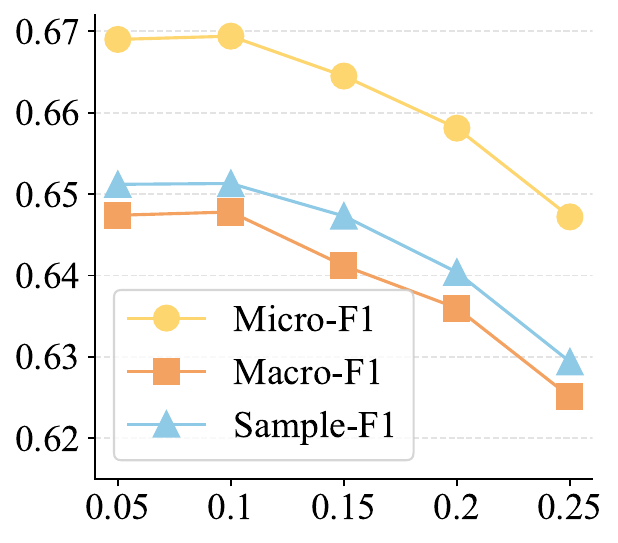} }
    \subfloat[Meta Path]{ \includegraphics[width=0.31\linewidth]{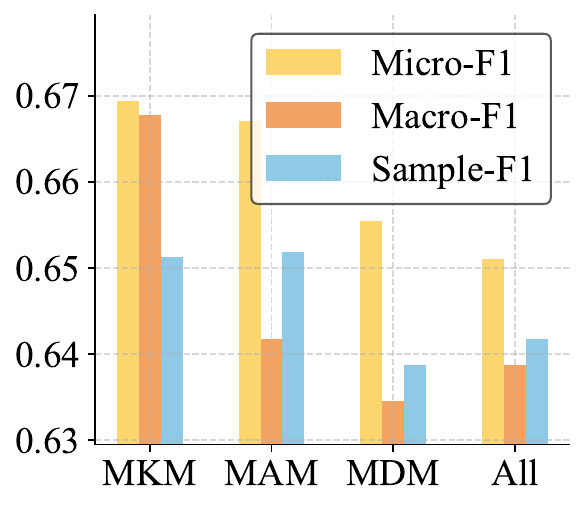} }
     \subfloat[Hidden Dimension]{ \includegraphics[width=0.31\linewidth]{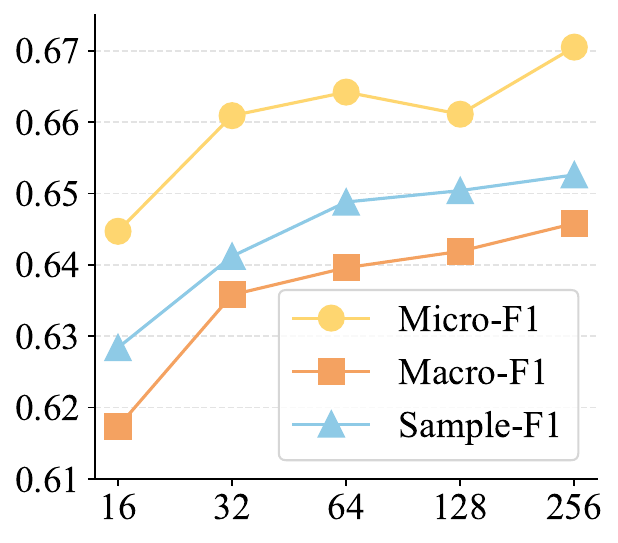} }
    \caption{Parameter analysis of FOCAL.} 
    \label{fig:exp_param}
\end{figure}

\section{Related work}
Multi-label node classification on \emph{homogeneous} graphs has been widely studied, where a central idea is to improve prediction by explicitly modeling label dependencies and label-aware representations (e.g., ML-GCN~\cite{gao2019semi}, LANC~\cite{zhou2021multi}, LARN~\cite{xiao2022semantic}, CorGCN~\cite{bei2025correlation}, and LIP~\cite{sunmulti}). However, these methods are primarily developed under homogeneous assumptions and do not directly address multi-typed nodes/relations or heterogeneous semantic interactions. In parallel, heterogeneous GNNs learn representations via meta-path/meta-graph semantics or relation-aware message passing (e.g., HAN~\cite{wang2019heterogeneous}, MAGNN~\cite{fu2020magnn}, GTN~\cite{yun2019graph}, HGT~\cite{hu2020heterogeneous}, and SimpleHGN~\cite{lv2021we}). While they can be \emph{adapted} to multi-label prediction by replacing the final softmax with sigmoid and using a binary multi-label loss, such a direct adaptation typically lacks explicit mechanisms for modeling label correlations and may not capture how heterogeneous semantics jointly determine multiple labels. Research that targets \emph{heterogeneous multi-label} learning is comparatively scarce: early approaches predate modern GNNs and rely on heuristic propagation or shallow transductive embeddings~\cite{zhou2014activity,dos2016multilabel}, and more recent work is often task-specific (e.g., Persona/TriPer~\cite{gupta2025persona}). Overall, general-purpose solutions that jointly model heterogeneous interactions and multi-label dependencies remain under-explored.
Detailed discussions are provided in Appendix \ref{sec:appendix_related_work}.

\section{Conclusion}
This paper proposes FOCAL, a role-separated attention framework that disentangles semantic coverage and semantic anchoring. FOCAL combines a coverage-oriented attention branch (COA) for unconstrained heterogeneous context aggregation with an anchoring-oriented attention branch (AOA) restricted to meta-path-induced primary neighborhoods, and integrates them via role-aware gated fusion and semantic-preserving residual propagation. 
Our theoretical analysis and experimental results further indicates that the proposed FOCAL achieves a better performance than other state-of-the-art models.

\clearpage

\bibliography{cite}
\bibliographystyle{icml2026}

\newpage
\appendix
\onecolumn

\section{Full Statements and Proofs of Section 3}

\subsection{Full Statements and Proofs of Theorem~\ref{lem:softmax_attention_dilution}}\label{sec:appendix_attn_dilution}

\begin{theorem}[Attention dilution of heterogeneous graph]\label{lem:attention_dilution}
Fix a target node \(t\) whose incoming neighbors are partitioned into a primary set \(\mathcal{S}^*\) and a secondary set \(\mathcal{S}^c\), with
\(|\mathcal{S}^*|=n^*\) and \(|\mathcal{S}^c|=m\).
Consider a softmax attention mechanism
\begin{equation}
\alpha_{t,s}=\frac{\exp(E_{t,s})}{\sum_{u\in\mathcal{S}^*\cup\mathcal{S}^c}\exp(E_{t,u})}.
\end{equation}

Define the total attention mass on the primary set
\begin{equation}
A^*(t)
:= \sum_{s\in\mathcal{S}^*}\alpha_{t,s}
= \frac{\sum_{i=1}^{n^*}\exp(E_i^*)}
      {\sum_{i=1}^{n^*}\exp(E_i^*) + \sum_{j=1}^{m}\exp(E_j)} .
\end{equation}
Let $X_{n^*}:=\sum_{i=1}^{n^*}\exp(E_i^*)$, $Y_m:=\sum_{j=1}^{m}\exp(E_j)$.
Assume that as \(m\to\infty\) (allowing \(n^*=n^*(m)\) to depend on \(m\)):
\begin{itemize}
\item[(\romannumeral 1)] \(\frac{1}{m}Y_m\xrightarrow[]{\mathrm{a.s.}}\mu\) for some \(\mu\in(0,\infty)\).
\item[(\romannumeral 2)] Either (a) \(\frac{1}{n^*}X_{n^*}\xrightarrow[]{\mathrm{a.s.}}\mu^*\) for some \(\mu^*\in(0,\infty)\), or (b) \(n^*\) is fixed (independent of \(m\)) and \(X_{n^*}<\infty\) almost surely.
\item[(\romannumeral 3)] In case (a), \(\mu^*/\mu=O(1)\) independent of \(m\).
\end{itemize}

Then in case (a),
$A^*(t)=\frac{n^*\mu^*}{n^*\mu^*+m\mu}+ o_{\mathrm{a.s.}}(1)$.
In particular, if \(m/n^*\to\infty\), then \(A^*(t)\to 0\) almost surely and
$A^*(t)=O\!\left(\frac{n^*}{m}\right)$.
In case (b), \(A^*(t)\to 0\) almost surely and \(A^*(t)=O(1/m)\).
\end{theorem}

\begin{proof}
By definition,
\begin{equation}
A^*(t)=\frac{X_{n^*}}{X_{n^*}+Y_m}.
\end{equation}

\noindent\textbf{Case (a).}
Divide numerator and denominator by \(m\):
\begin{equation}
A^*(t)
=
\frac{\frac{X_{n^*}}{m}}
{\frac{X_{n^*}}{m}+\frac{Y_m}{m}}
=
\frac{\frac{n^*}{m}\cdot \frac{X_{n^*}}{n^*}}
{\frac{n^*}{m}\cdot \frac{X_{n^*}}{n^*}+\frac{Y_m}{m}}.
\end{equation}
Under assumption (\romannumeral 1) and (a) in (\romannumeral 2) mentioned above, almost surely,
\begin{equation}
\frac{X_{n^*}}{n^*}\to \mu^*,
\qquad
\frac{Y_m}{m}\to \mu.
\end{equation}
Since \(\mu>0\), by the continuous mapping theorem,
\begin{equation}
A^*(t)
=
\frac{\frac{n^*}{m}\mu^*}{\frac{n^*}{m}\mu^*+\mu}
+ o_{\mathrm{a.s.}}(1)
=
\frac{n^*\mu^*}{n^*\mu^*+m\mu}
+ o_{\mathrm{a.s.}}(1).
\end{equation}
If \(m/n^*\to\infty\), then the leading term converges to \(0\), hence \(A^*(t)\to 0\) almost surely. Moreover,
\begin{equation}
\frac{n^*\mu^*}{n^*\mu^*+m\mu}
\le
\frac{n^*\mu^*}{m\mu}
=
O\!\left(\frac{n^*}{m}\right),
\end{equation}
and (\romannumeral 3) ensures \(\mu^*/\mu\) does not grow with \(m\).

\noindent\textbf{Case (b).}
Here \(X_{n^*}\) is a fixed (in \(m\)) almost surely finite random variable. Under (\romannumeral 1), \(Y_m/m\to\mu>0\) almost surely, so \(Y_m\to\infty\) almost surely and thus
\begin{equation}
A^*(t)=\frac{X_{n^*}}{X_{n^*}+Y_m}\xrightarrow[]{\mathrm{a.s.}}0.
\end{equation}
Furthermore, on any sample path where \(Y_m/m\to\mu\), for all sufficiently large \(m\) we have \(Y_m\ge (\mu/2)m\), hence
\begin{equation}
A^*(t)\le \frac{X_{n^*}}{Y_m}\le \frac{2X_{n^*}}{\mu}\cdot \frac{1}{m}=O\!\left(\frac{1}{m}\right).
\end{equation}
\end{proof}

\subsection{Full Statements and Proofs of Theorem~\ref{lem:multilabel_loss_grad_amp}}\label{app:multilabel_amp_full}

\begin{theorem}[Multi-label amplification of attention dilution]
\label{lem:multilabel_amplification}
Fix a node $t$ with one-step attention aggregation
$h_t=\sum_{s\in\mathcal{S}^*\cup\mathcal{S}^c}\alpha_{t,s}\,m_{s\to t}$.
Let $A^*(t):=\sum_{s\in\mathcal{S}^*}\alpha_{t,s}$ and
$h_t^*:=\sum_{s\in\mathcal{S}^*}\alpha_{t,s}\,m_{s\to t}$.
Assume (i) $\|m_{s\to t}\|\le M$ for all $s\in\mathcal{S}^*$, and
(ii) for the classifier $z_k=w_k^\top h_t$ and each positive label $k\in\mathcal{P}(t)$,
there exist constants $a>0$ and $b\le aM$ such that $z_k \le a\|h_t^*\|+b$.
Then, letting $L:=|\mathcal{P}(t)|$, the positive-label BCE term satisfies
\begin{equation}
\sum_{k\in\mathcal{P}(t)} -\log \sigma(z_k)
\;\ge\;
L\cdot\Bigl(-\log \sigma\!\bigl(aMA^*(t)+b\bigr)\Bigr).
\end{equation}
\end{theorem}

\begin{proof}
We first upper bound the magnitude of the primary component $h_t^*$ in terms of the primary attention mass $A^*(t)$.
By the triangle inequality and assumption (i),
\begin{equation}
\|h_t^*\|
=
\left\|\sum_{s\in\mathcal{S}^*}\alpha_{t,s}\,m_{s\to t}\right\|
\le
\sum_{s\in\mathcal{S}^*}\alpha_{t,s}\,\|m_{s\to t}\|
\le
\sum_{s\in\mathcal{S}^*}\alpha_{t,s}\,M
=
MA^*(t).
\end{equation}
Next, apply assumption (ii): for each $k\in\mathcal{P}(t)$,
\begin{equation}
z_k \le a\|h_t^*\|+b \le aMA^*(t)+b.
\end{equation}
Now consider $\ell(x):=-\log\sigma(x)$. Since
\begin{equation}
\ell'(x)= -\frac{\sigma'(x)}{\sigma(x)} = -\bigl(1-\sigma(x)\bigr) < 0,
\end{equation}
$\ell(\cdot)$ is monotonically decreasing. Therefore, from $z_k \le aMA^*(t)+b$ we get
\begin{equation}
-\log\sigma(z_k)=\ell(z_k)\ge \ell\bigl(aMA^*(t)+b\bigr)
= -\log\sigma\!\bigl(aMA^*(t)+b\bigr).
\end{equation}
Summing the above inequality over all positive labels $k\in\mathcal{P}(t)$ yields
\begin{equation}
\sum_{k\in\mathcal{P}(t)} -\log \sigma(z_k)
\ge
\sum_{k\in\mathcal{P}(t)} \Bigl[-\log\sigma\!\bigl(aMA^*(t)+b\bigr)\Bigr]
=
L\cdot\Bigl(-\log\sigma\!\bigl(aMA^*(t)+b\bigr)\Bigr),
\end{equation}
which proves the stated lower bound.

Finally, note that the bound is (i) monotonically degrading in $A^*(t)$ because
$-\log\sigma(\cdot)$ is decreasing and $aM>0$, and (ii) linear in $L$ by construction.
Moreover, combining with Theorem~\ref{lem:attention_dilution}, if $A^*(t)\to 0$ ,
then by continuity of $-\log\sigma(\cdot)$ the lower bound converges to
\(
L\cdot\bigl(-\log\sigma(b)\bigr).
\)
\end{proof}

\begin{proposition}
\label{prop:gradient_attenuation}
Under the assumptions of Theorem~\ref{lem:attention_dilution} and Theorem~\ref{lem:multilabel_amplification},
consider the multi-label BCE loss \(\mathcal{L}_{\mathrm{ML}}(t)\).
Assume that for all \(k\in\mathcal{P}(t)\) and all \(s\in\mathcal{S}^*\), there exists a constant \(C>0\) independent of \(m\) such that
$\left\|\nabla_{m_{s\to t}} z_k \right\| \le C\,\alpha_{t,s}$.
Then the gradient norm of the multi-label loss with respect to the collection of primary messages satisfies
$\left\|\nabla_{\{m_{s\to t}:s\in\mathcal{S}^*\}} \mathcal{L}_{\mathrm{ML}}(t)\right\|=O\!\left(L \cdot A^*(t)\right)$.
Consequently, in the regime \(m/n^*\to\infty\) (so that \(A^*(t)\to 0\)),
$\left\|\nabla_{\{m_{s\to t}:s\in\mathcal{S}^*\}} \mathcal{L}_{\mathrm{ML}}(t)\right\| =O\!\left(L \cdot \frac{n^*}{m}\right)$.
In particular, if additionally \(n^*=O(1)\), then the above becomes \(O(L/m)\).
\end{proposition}

\begin{proof}
For each label \(k\), the gradient of the loss with respect to the logit is
\begin{equation}
\frac{\partial \ell_k}{\partial z_k}
=
\sigma(z_k) - y_k,
\end{equation}
with \(|\sigma(z_k)-y_k|\le 1\).
By the chain rule, for any \(s\in\mathcal{S}^*\),
\begin{equation}
\nabla_{m_{s\to t}} \ell_k
=
(\sigma(z_k)-y_k)\,\nabla_{m_{s\to t}} z_k,
\end{equation}
and therefore
\begin{equation}
\|\nabla_{m_{s\to t}} \ell_k\|
\le
\|\nabla_{m_{s\to t}} z_k\|
\le
C\,\alpha_{t,s}.
\end{equation}
Summing over all \(k\in\mathcal{P}(t)\) and \(s\in\mathcal{S}^*\) yields
\begin{equation}
\left\|\nabla_{\{m_{s\to t}:s\in\mathcal{S}^*\}} \mathcal{L}_{\mathrm{ML}}(t)\right\|
\le
\sum_{k\in\mathcal{P}(t)}\sum_{s\in\mathcal{S}^*}
\|\nabla_{m_{s\to t}} \ell_k\|
\le
\sum_{k\in\mathcal{P}(t)}\sum_{s\in\mathcal{S}^*}
C\,\alpha_{t,s}
=
C\,L\,A^*(t),
\end{equation}
which proves the \(O(LA^*(t))\) claim. The final scaling statements follow from Theorem~\ref{lem:attention_dilution}.
\end{proof}

\subsection{Full Statement and Proof of Theorem~\ref{lem:meta_path_subspace_dilution_main}}
\label{app:meta_path_tradeoff_full}

\begin{theorem}[The coverage-dilution dilemma of meta-path-based attention]
\label{thm:meta_path_tradeoff}
Consider a model with a predefined meta-path set \(\mathcal{M}\),
whose representation takes the form
$h_v = \sum_{\mathcal{P}\in\mathcal{M}} \beta_{\mathcal{P}}(v)\, z_v^{\mathcal{P}}$. Then the following properties hold:
\begin{enumerate}
    \item 
    For all nodes \(v\), the representation \(h_v\) lies in the meta-path-induced subspace
    \begin{equation}
    \mathcal{H}_{\mathcal{M}}
    :=
    \mathrm{span}\!\left(\{ z_v^{\mathcal{P}} \mid \mathcal{P}\in\mathcal{M} \}\right).
    \end{equation}

    \item 
    Let \(\mathcal{M}^*\subset\mathcal{M}\) denote a set of primary meta-paths.
    Under mild non-degeneracy conditions on semantic scores,
    the total attention mass assigned to \(\mathcal{M}^*\),
    \begin{equation}
    B^*(v)
    :=
    \sum_{\mathcal{P}\in\mathcal{M}^*}\beta_{\mathcal{P}}(v),
    \end{equation}
    satisfies
    \begin{equation}
    B^*(v)
    =
    O\!\left(\frac{|\mathcal{M}^*|}{|\mathcal{M}|}\right).
    \end{equation}
\end{enumerate}
\end{theorem}

\begin{proof}
We consider two complementary cases.

\textbf{Case 1: Insufficient meta-path coverage.}
Suppose that the meta-path set \(\mathcal{M}\) is restricted to a limited subset
that does not cover all task-relevant relations or node types.
Let \(\xi(v)\) denote a signal contributing to the target label \(y(v)\)
that is not fully observable through any \(\mathcal{P}\in\mathcal{M}\).

We formalize ``coverage'' via the best recoverability of \(\xi(v)\) from the meta-path-induced summaries:
\begin{equation}
\epsilon_{\mathcal{M}}
:=
\inf_{f}\ \mathbb{E}\!\Big[\big\|\xi(v)-f(\{z_v^{\mathcal{P}}\}_{\mathcal{P}\in\mathcal{M}})\big\| \Big],
\end{equation}
where the infimum is taken over measurable functions \(f\).
When \(\mathcal{M}\) becomes less expressive, \(\epsilon_{\mathcal{M}}\) increases.

Since the representation \(h_v\) is entirely determined by \(\{z_v^{\mathcal{P}} : \mathcal{P}\in\mathcal{M}\}\),
the influence of \(\xi(v)\) on \(h_v\) is necessarily mediated and attenuated.
In particular, for node pairs \((v_1,v_2)\) whose differences lie predominantly in \(\xi(v)\),
the induced representations satisfy $\|h_{v_1} - h_{v_2}\| \le C \cdot \epsilon_{\mathcal{M}}$, for some constant \(C\).
As a result, the discriminative margin achievable by classifiers operating on \(h_v\) is upper bounded,
which limits the expressive capacity under insufficient meta-path coverage.

\textbf{Case 2: Excessive meta-path coverage.}
Suppose that the meta-path set \(\mathcal{M}\) is expanded to include a large number
of meta-paths in order to increase semantic coverage.
Let \(\mathcal{M}^*\subset\mathcal{M}\) denote the subset of critical meta-paths,
with \(|\mathcal{M}^*|\ll |\mathcal{M}|\).
Recall the semantic-level attention weights
\begin{equation}
\beta_{\mathcal{P}}(v)
=
\frac{\exp(s_{\mathcal{P}}(v))}{\sum_{\mathcal{Q}\in\mathcal{M}}\exp(s_{\mathcal{Q}}(v))}.
\end{equation}
We consider the total semantic attention mass assigned to critical meta-paths,
\begin{equation}
B^*(v)
:=
\sum_{\mathcal{P}\in\mathcal{M}^*}\beta_{\mathcal{P}}(v)
=
\frac{\sum_{\mathcal{P}\in\mathcal{M}^*}\exp(s_{\mathcal{P}}(v))}
{\sum_{\mathcal{Q}\in\mathcal{M}}\exp(s_{\mathcal{Q}}(v))}.
\end{equation}

To obtain a scaling bound in \(|\mathcal{M}|\), we impose two mild conditions:
(i) critical meta-paths do not gain an unbounded score advantage as more weakly-informative meta-paths are added, and
(ii) the added meta-paths contribute a non-vanishing mass to the softmax normalizer.

We assume: 
\begin{enumerate}
\item 
\label{ass:bounded_critical_scores}
There exists a constant \(\overline{s}<\infty\) such that for all nodes \(v\),
\(
\max_{\mathcal{P}\in\mathcal{M}^*} s_{\mathcal{P}}(v) \le \overline{s}.
\)
\item
\label{ass:nondegenerate_noncritical_mass}
There exist constants \(c\in(0,1)\) and \(\underline{s}\in\mathbb{R}\) such that for all nodes \(v\),
at least \(c|\mathcal{M}|\) meta-paths in \(\mathcal{M}\setminus\mathcal{M}^*\) satisfy
\(
s_{\mathcal{Q}}(v) \ge \underline{s}.
\)
\end{enumerate}
Under \ref{ass:bounded_critical_scores}, we upper bound the numerator:
\begin{equation}
\sum_{\mathcal{P}\in\mathcal{M}^*}\exp(s_{\mathcal{P}}(v))
\le
|\mathcal{M}^*| \exp\!\left(\max_{\mathcal{P}\in\mathcal{M}^*} s_{\mathcal{P}}(v)\right)
\le
|\mathcal{M}^*| e^{\overline{s}}.
\end{equation}
Under \ref{ass:nondegenerate_noncritical_mass}, we lower bound the denominator by restricting the sum:
\begin{equation}
\sum_{\mathcal{Q}\in\mathcal{M}}\exp(s_{\mathcal{Q}}(v))
\ge
\sum_{\mathcal{Q}\in\mathcal{M}\setminus\mathcal{M}^*}\exp(s_{\mathcal{Q}}(v))
\ge
c|\mathcal{M}| \, e^{\underline{s}}.
\end{equation}
Combining the two bounds yields
\begin{equation}
B^*(v)
\le
\frac{|\mathcal{M}^*| e^{\overline{s}}}{c|\mathcal{M}| e^{\underline{s}}}
=
\frac{e^{\overline{s}-\underline{s}}}{c}\cdot \frac{|\mathcal{M}^*|}{|\mathcal{M}|}
=
O\!\left(\frac{|\mathcal{M}^*|}{|\mathcal{M}|}\right),
\end{equation}
where the hidden constant depends only on \(\overline{s}-\underline{s}\) and \(c\).
This shows that when \(\mathcal{M}\) grows by including many weakly informative meta-paths,
the semantic attention mass allocated to critical semantics is necessarily diluted,
thereby limiting the effective contribution of \(\mathcal{M}^*\) to the final representation.
In both cases, the limitation arises from the predefined, discrete nature of the meta-path set \(\mathcal{M}\).
\end{proof}

\subsection{Full Statements and Proofs of Theorem~\ref{lem:multilabel_amplification_main}}
\label{app:multilabel_attenuation_full}

\begin{theorem}[Positive-loss floor grows with the number of positives]
\label{prop:multilabel_loss_lowerbound}
Assume \(\|w_k\|\le W\) for all \(k\) and \(\|h_v^*\|\le H\) and setting \(r_v=0\) for
\(h_v = B^*(v)\, h_v^* + r_v\), the positive-part loss satisfies
\begin{equation}
\label{eq:pos_loss_lowerbound}
\begin{aligned}
\sum_{k\in \mathcal{P}(v)} \log\!\big(1+\exp(-z_{v,k})\big)
&\ge - \frac{|\mathcal{P}(v)|}{2}\, W H \, B^*(v)  + |\mathcal{P}(v)|\cdot \log 2 .
\end{aligned}
\end{equation}
As \(B^*(v)\to 0\), the lower bound approaches \(|\mathcal{P}(v)|\log 2\) up to \(O(B^*(v))\).
\end{theorem}

\begin{proof}
In the simplified setting \(r_v=0\), we have \(h_v = B^*(v)\,h_v^*\).
For any \(k\in\mathcal{P}(v)\),
\begin{equation}
|z_{v,k}| = |w_k^\top h_v|
\le \|w_k\|\,\|h_v\|
\le W \cdot B^*(v)\,\|h_v^*\|
\le W H\, B^*(v).
\end{equation}
For any \(z\in\mathbb{R}\), the function \(g(z)=\log(1+\exp(-z))\) is convex and satisfies
\(g(z)\ge g(0) + g'(0)z = \log 2 - z/2\).
Hence,
\begin{equation}
\log(1+\exp(-z_{v,k})) \ge \log 2 - \frac{z_{v,k}}{2}
\ge \log 2 - \frac{|z_{v,k}|}{2}
\ge \log 2 - \frac{WH\,B^*(v)}{2}.
\end{equation}
Summing over \(k\in\mathcal{P}(v)\) yields Eq.~\eqref{eq:pos_loss_lowerbound}.
\end{proof}

\begin{proposition}
\label{prop:multilabel_grad_attenuation}
Assume \(\|w_k\|\le W\) for all \(k\).
Let 
\(
\mathcal{L}_{\mathrm{ML}}(v) = \sum_{k=1}^K \Big(
- y_{v,k}\log \sigma(z_{v,k})
- (1-y_{v,k})\log\big(1-\sigma(z_{v,k})\big)
\Big),
\)
and 
\(\nabla_{h_v}\mathcal{L}_{\mathrm{ML}}(v)\) denote the gradient of the multi-label loss with respect to \(h_v\).
Then
\begin{equation}
\label{eq:grad_bound}
\Big\|\nabla_{h_v}\mathcal{L}_{\mathrm{ML}}(v)\Big\|
\;\le\;
W \sum_{k=1}^K \big|\sigma(z_{v,k})-y_{v,k}\big|.
\end{equation}
\end{proposition}

\begin{proof}
For binary cross-entropy with logits, the per-label derivative satisfies
\(\nabla_{h_v}\ell_k(v) = (\sigma(z_{v,k})-y_{v,k})\, w_k\),
where \(z_{v,k}=w_k^\top h_v\).
Thus
\begin{equation}
\nabla_{h_v}\mathcal{L}_{\mathrm{ML}}(v)
=
\sum_{k=1}^K (\sigma(z_{v,k})-y_{v,k})\, w_k.
\end{equation}
By the triangle inequality and \(\|w_k\|\le W\),
\begin{equation}
\Big\|\nabla_{h_v}\mathcal{L}_{\mathrm{ML}}(v)\Big\|
\le
\sum_{k=1}^K \big|\sigma(z_{v,k})-y_{v,k}\big|\,\|w_k\|
\le
W \sum_{k=1}^K \big|\sigma(z_{v,k})-y_{v,k}\big|.
\end{equation}
This is Eq.~\eqref{eq:grad_bound}.
\end{proof}

\begin{proposition}
\label{prop:pos_error_lower_small_logits}
Fix any \(\delta\ge 0\). If \(|z_{v,k}|\le \delta\) and \(y_{v,k}=1\), then
\begin{equation}
\label{eq:pos_error_lower_small_logits}
\big|\sigma(z_{v,k})-1\big|
=
1-\sigma(z_{v,k})
\;\ge\;
1-\sigma(\delta)
=
\sigma(-\delta).
\end{equation}
Consequently, if \(|z_{v,k}|\le \delta\) holds for all \(k\in\mathcal{P}(v)\), then
\begin{equation}
\label{eq:sum_pos_error_lower_small_logits}
\sum_{k=1}^K \big|\sigma(z_{v,k})-y_{v,k}\big|
\;\ge\;
\sum_{k\in\mathcal{P}(v)} \big|\sigma(z_{v,k})-1\big|
\;\ge\;
|\mathcal{P}(v)|\,\sigma(-\delta).
\end{equation}
\end{proposition}

\begin{proof}
Since \(\sigma(\cdot)\) is increasing and \(|z_{v,k}|\le \delta\) implies \(z_{v,k}\le \delta\), we have
\begin{equation}
1-\sigma(z_{v,k})
\ge
1-\sigma(\delta).
\end{equation}
Using the identity \(1-\sigma(t)=\sigma(-t)\) yields Eq.~\eqref{eq:pos_error_lower_small_logits}.
Summing over \(k\in\mathcal{P}(v)\) gives Eq.~\eqref{eq:sum_pos_error_lower_small_logits}.
\end{proof}

Since \(\sigma'(t)=\sigma(t)(1-\sigma(t))\le \tfrac{1}{4}\) for all \(t\),
\(\sigma(\cdot)\) is \(\tfrac{1}{4}\)-Lipschitz. Hence, for \(\delta\ge 0\),
\begin{equation}
\label{eq:sigmoid_linear_lower_optional}
\sigma(-\delta)
\ge
\sigma(0)-\frac{\delta}{4}
=
\frac{1}{2}-\frac{\delta}{4}.
\end{equation}

Given \(h_v = B^*(v)\, h_v^* + r_v\) in the simplified setting \(r_v=0\), we have \(h_v=B^*(v)h_v^*\).
If additionally \(\|h_v^*\|\le H\) and \(\|w_k\|\le W\), then for any label \(k\),
\begin{equation}
|z_{v,k}| = |w_k^\top h_v|
\le \|w_k\|\,\|h_v\|
\le W \cdot B^*(v)\,\|h_v^*\|
\le WH\,B^*(v).
\end{equation}
Thus, for positive labels \(k\in\mathcal{P}(v)\), Proposition~\ref{prop:pos_error_lower_small_logits} applies with
\(\delta := WH\,B^*(v)\), yielding
\begin{equation}
\sum_{k=1}^K \big|\sigma(z_{v,k})-y_{v,k}\big|
\ge
|\mathcal{P}(v)|\,\sigma\big(-WH\,B^*(v)\big).
\end{equation}
When \(B^*(v)\) is small, \(\sigma(-WH\,B^*(v))\) is a positive constant close to \(1/2\)
, e.g., Eq.~\eqref{eq:sigmoid_linear_lower_optional} gives
\begin{equation}
\sigma(-WH\,B^*(v)) \ge \tfrac12 - \tfrac{WH\,B^*(v)}{4}
\end{equation}
Therefore, the cumulative positive-label term grows proportionally with \(|\mathcal{P}(v)|\),
which is the multi-label amplification effect summarized in Theorem~\ref{lem:multilabel_amplification_main}.

\section{Full Statements and Proofs of Section~\ref{Methodology}}\label{sec:appendix}

\subsection{AOA with Multiple Meta-Paths}\label{app:aoa_multi_metapath}
Let \(\mathcal{M}\) denote a set of predefined meta-paths.
For a target node \(t\), each meta-path \( \mathcal{P}\in\mathcal{M}\) induces a neighborhood \(\mathcal{N}_{ \mathcal{P}}(t)\).
AOA then performs hierarchical aggregation: (i) node-level attention within each \(\mathcal{N}_{ \mathcal{P}}(t)\), followed by (ii) semantic-level attention over \(\mathcal{M}\).

For each \( \mathcal{P}\in\mathcal{M}\) and head \(i\in\{1,\dots,h\}\), we compute the unnormalized attention score for \(s\in\mathcal{N}_{ \mathcal{P}}(t)\) as
\begin{equation}
e^{ \mathcal{P},i}_{t,s}
=
\mathrm{LeakyReLU}\!\left(
\left(\mathbf{a}^{ \mathcal{P},i}\right)^\top
\left[
W^{ \mathcal{P},i} h^{(l)}(t)
\,\|\,
W^{ \mathcal{P},i} h^{(l)}(s)
\right]
\right),
\end{equation}
and normalize it within \(\mathcal{N}_{ \mathcal{P}}(t)\):
\begin{equation}
\alpha^{ \mathcal{P},i}_{t,s}
=
\operatorname{softmax}_{s\in\mathcal{N}_{ \mathcal{P}}(t)}
\left(e^{ \mathcal{P},i}_{t,s}\right).
\end{equation}
The meta-path-specific embedding is then
\begin{equation}
z_{ \mathcal{P}}(t)
=
\big\|_{i=1}^{h}
\sum_{s\in\mathcal{N}_{ \mathcal{P}}(t)}
\alpha^{ \mathcal{P},i}_{t,s}\;W^{ \mathcal{P},i} h^{(l)}(s).
\end{equation}

To combine embeddings produced by different meta-paths, we compute a meta-path importance score:
\begin{equation}
u_{ \mathcal{P}}(t)
=
\mathbf{q}^{\top}\tanh\!\left(W_{s}\,z_{ \mathcal{P}}(t)+\mathbf{b}_{s}\right),
\end{equation}
and normalize over \( \mathcal{P}\in\mathcal{M}\):
\begin{equation}
\beta_{ \mathcal{P}}(t)
=
\operatorname{softmax}_{ \mathcal{P}\in\mathcal{M}}\left(u_{ \mathcal{P}}(t)\right).
\end{equation}
Finally, the anchored representation is obtained by
\begin{equation}
h^{(l+1)}_{\mathrm{AOA}}(t)
=
\sum_{ \mathcal{P}\in\mathcal{M}}
\beta_{ \mathcal{P}}(t)\,z_{ \mathcal{P}}(t).
\end{equation}

When \(|\mathcal{M}|=1\), semantic-level attention degenerates to a constant weight, i.e., \(\beta_{ \mathcal{P}}(t)=1\), and the above formulation reduces to the single meta-path AOA used in the main text.

\subsection{Full Statement and Proof of Theorem~\ref{thm:rosa_main}}\label{app:theory}

We provide formal statements and proofs supporting that FOCAL addresses three issues discussed in the main text:
(i) neighbor-level dilution in attention-based aggregation,
(ii) visibility preservation of secondary information under fusion,
(iii) semantic-layer dilution in meta-path-based semantic attention.

\textbf{Notations.}
We follow the notation of Section~\ref{Preliminaries}. 
In particular, $\mathcal{R}$ denotes the set of relation types, which we partition into 
primary types $\mathcal{R}_p$ and secondary types $\mathcal{R}_s$ 
with $\mathcal{R}_p \cap \mathcal{R}_s = \varnothing$ and $\mathcal{R}_p \cup \mathcal{R}_s = \mathcal{R}$.
Let $\mathcal{M}$ be a pre-specified set of meta-paths over relation types. 
We split it into primary meta-paths $\mathcal{M}_p$ and secondary meta-paths $\mathcal{M}_s$:
\begin{equation}
    \mathcal{M}_p := \{\Phi \in \mathcal{M} : \text{all relation types on } \Phi \text{ lie in } \mathcal{R}_p\}, 
    \qquad 
    \mathcal{M}_s := \mathcal{M} \setminus \mathcal{M}_p, 
    \qquad 
    \mathcal{M} = \mathcal{M}_p \cup \mathcal{M}_s.
\end{equation}

For each node $v$, we consider two branch representations:
(1)~$h_{\mathrm{COA}}(v) \in \mathbb{R}^d$, computed by attention-based aggregation over the 
full heterogeneous neighborhood (depending on both primary and secondary relations); 
and 
(2)~$h_{\mathrm{AOA}}(v) \in \mathbb{R}^d$, computed by meta-path-based aggregation restricted to 
the primary meta-paths $\mathcal{M}_p$.
FOCAL fuses the two branches using element-wise gating:
\begin{equation}\label{eq:app_fusion}
    h_{\mathrm{fuse}}(v) 
    = g_v^{(1)} \odot h_{\mathrm{COA}}(v) + g_v^{(2)} \odot h_{\mathrm{AOA}}(v), 
    \qquad g_v^{(1)}, g_v^{(2)} \in (0,1)^d,
\end{equation}
where $\odot$ denotes the Hadamard product and the gates are functions of both branch vectors, specifically, we have $g_v^{(1)} = f_1(h_{\mathrm{COA}}(v), h_{\mathrm{AOA}}(v))$ and 
$g_v^{(2)} = f_2(h_{\mathrm{COA}}(v), h_{\mathrm{AOA}}(v))$.

\begin{definition}\label{def:dep}
Let $\mathcal{X}$ denote the collection of inputs that may affect the computation 
(node features, edges, parameters, etc.). 
The \emph{dependency set} of a representation $h(v)$ is defined as
\begin{equation}
    \mathrm{Dep}\bigl(h(v)\bigr) 
    := \bigl\{ x \in \mathcal{X} \;:\; 
       \exists\, \tilde{x} \neq x \;\text{s.t.}\; 
       h(v;\, \tilde{x},\, \mathcal{X}\setminus\{x\}) 
       \neq h(v;\, x,\, \mathcal{X}\setminus\{x\}) \bigr\}.
\end{equation}
\end{definition}
\textbf{Result 1: FOCAL protects the lower bound of primary semantics via AOA}\label{app:res1}
\begin{theorem}\label{thm:primary_lb}
There exists a constant \( \gamma > 0 \) such that

\begin{equation}
\|g^{(2)}_v \odot h_{\mathrm{AOA}}(v)\|_2 \geq \gamma \|h_{\mathrm{AOA}}(v)\|_2 > 0.
\end{equation}

This means that the AOA contribution in the fused representation \( h_{\mathrm{fuse}}(v) \) will be non-vanishing.
\end{theorem}

\begin{proof}

The gating vector \( g^{(2)}_v \) is computed via a Sigmoid function, and thus \( g^{(2)}_{v,k} \in (0, 1) \) for all coordinates \( k \). Consequently, we have:
\(
\forall k,
\)
\(
g^{(2)}_{v,k} \geq \gamma,
\)
for some constant \( \gamma \in (0, 1) \), ensuring that the gating vector \( g^{(2)}_v \) does not approach zero.

Now, we prove that \( g^{(2)}_v \odot h_{\mathrm{AOA}}(v) \) has a non-zero lower bound. For each coordinate \( k \), we have:
\begin{equation}
|g^{(2)}_{v,k} h_{\mathrm{AOA},k}(v)| \geq \gamma |h_{\mathrm{AOA},k}(v)|.
\end{equation}
Summing over all coordinates and taking the square root, we get:
\begin{equation}
\| g^{(2)}_v \odot h_{\mathrm{AOA}}(v) \|_2 \geq \gamma \| h_{\mathrm{AOA}}(v) \|_2.
\end{equation}
Since \( \| h_{\mathrm{AOA}}(v) \|_2 > 0 \), we conclude that:
\begin{equation}
\| g^{(2)}_v \odot h_{\mathrm{AOA}}(v) \|_2 \geq \gamma \| h_{\mathrm{AOA}}(v) \|_2 > 0.
\end{equation}

Finally, we consider the fused representation:
\begin{equation}
h_{\mathrm{fuse}}(v) = g^{(1)}_v \odot h_{\mathrm{COA}}(v) + g^{(2)}_v \odot h_{\mathrm{AOA}}(v).
\end{equation}
The term \( g^{(2)}_v \odot h_{\mathrm{AOA}}(v) \) ensures that the primary semantics carried by the AOA branch is not suppressed by COA, guaranteeing that the contribution from the AOA branch remains non-vanishing in the fused representation.
Thus, FOCAL preserves the primary semantics by injecting a non-vanishing contribution through the AOA branch, even when COA is strong.
\end{proof}

\textbf{Result 2: FOCAL preserves the coverage of secondary semantics via COA}\label{app:res2}
\begin{theorem}
\label{the:dep_subset}
Let $h_{\mathrm{fuse}}$ be defined in \eqref{eq:app_fusion}. Then, FOCAL preserves the coverage of secondary semantics via COA, and we have:
\begin{equation}\label{eq:dep_subset}
\mathrm{Dep}\bigl(h_{\mathrm{fuse}}(v)\bigr)
=
\mathrm{Dep}\bigl(h_{\mathrm{COA}}(v)\bigr)
\cup
\mathrm{Dep}\bigl(h_{\mathrm{AOA}}(v)\bigr).
\end{equation}
\end{theorem}

\begin{proof}
For the inclusion ``$\subseteq$'', \(h_{\mathrm{fuse}}(v)\) is computed from \(h_{\mathrm{COA}}(v)\) and \(h_{\mathrm{AOA}}(v)\), along with their gates. These gates are functions of the same pair, meaning that no input outside the union of the dependency sets of \(h_{\mathrm{COA}}(v)\) and \(h_{\mathrm{AOA}}(v)\) can affect \(h_{\mathrm{fuse}}(v)\). Thus, fusion retains the visibility and coverage of all secondary information that was initially accessible through COA.

For the reverse inclusion ``$\supseteq$'': We assume that the fusion map \( F(a,b) := f_1(a,b) \odot a + f_2(a,b) \odot b \) is \emph{non-degenerate}. This means that for any input \( x \in \mathcal{X} \), if changing \( x \) alters \( a = h_{\mathrm{COA}}(v) \) while keeping \( b = h_{\mathrm{AOA}}(v) \) fixed, then this change will also affect the fused representation \( h_{\mathrm{fuse}}(v) \), and symmetrically for changes in \( b \) while \( a \) is fixed.

Now, take \( x \in \mathrm{Dep}(h_{\mathrm{COA}}(v)) \). This means there is a change in \( x \) that alters \( h_{\mathrm{COA}}(v) \). By the non-degeneracy assumption, this change will also affect the fused representation \( h_{\mathrm{fuse}}(v) \), meaning that secondary semantics (those visible to COA) are preserved in the fused output. Thus, \( x \) must also be in \( \mathrm{Dep}(h_{\mathrm{fuse}}(v)) \). The same reasoning applies for \( x \in \mathrm{Dep}(h_{\mathrm{AOA}}(v)) \), confirming that fusion maintains the full coverage of secondary information from COA.
\end{proof}

\textbf{Result 3: FOCAL resolves the coverage-dilution dilemma through role separation between anchoring and coverage}\label{app:res3}

\begin{theorem}\label{thm:decoupling}
FOCAL achieves structural decoupling of anti-dilution and coverage. Specifically:

\textbf{(a) }
For any $\mathcal{M}_s'\supseteq\mathcal{M}_s$, let $\mathcal{M}'=\mathcal{M}_p\cup\mathcal{M}_s'$. Then:
\begin{equation}\label{a}
\alpha^{\mathrm{AOA},k}_{t,m}\!\bigl(\mathcal{M}'\bigr)=\alpha^{\mathrm{AOA},k}_{t,m}\!\bigl(\mathcal{M}\bigr),\quad\forall\,m\in\mathcal{M}_p,\;\forall\,k.
\end{equation}
\textbf{(b) }
Under mild non-degeneracy conditions:
\begin{equation}
\mathrm{Dep}(h^{(l)}_{\mathrm{fuse}},t)\supseteq\mathrm{Dep}(h^{(l)}_{\mathrm{COA}},t).
\end{equation}
\end{theorem}

\begin{proof}
\textbf{(a).}
The softmax normalization of AOA is restricted to $\mathcal{M}_p$:
\begin{equation}
\alpha^{\mathrm{AOA},k}_{t,m}(\mathcal{M}) = \frac{\exp(e^k_{t,m})}{\sum_{m'\in\mathcal{M}_p}\exp(e^k_{t,m'})},\quad m\in\mathcal{M}_p.
\end{equation}
Both the numerator and the denominator depend only on $\{e^k_{t,m'}\}_{m'\in\mathcal{M}_p}$. Expanding $\mathcal{M}_s$ to any $\mathcal{M}_s'\supseteq\mathcal{M}_s$ leaves every term unchanged, hence we have Equation \ref{a}. 
Because $h^{(l)}_{\mathrm{AOA}}(t)=f_{\mathrm{AOA}}\!\bigl(\{\mathbf{h}_m(t)\}_{m\in\mathcal{M}_p}\bigr)$ is fully determined by these attention weights and the meta-path representations within $\mathcal{M}_p$, it is invariant to $\mathcal{M}_s$.

The fusion gate $g^{(2)}_t=\sigma\bigl(W^{(2)}_g[h^{(l)}_{\mathrm{COA}}(t)\,\|\,h^{(l)}_{\mathrm{AOA}}(t)]+b^{(2)}_g\bigr)$ does depend on $h^{(l)}_{\mathrm{COA}}(t)$, which varies with $\mathcal{M}_s$. 
However, $g^{(i)}_t$ is the result of sigmoid function, $[g^{(2)}_t]_j=\sigma(\cdot)\in(0,1)$ for all $j$ and all parameter configurations, so the anti-dilution lower bound in Theorem~\ref{thm:primary_lb} is never violated.

\textbf{(b).}
For any $m\in\mathcal{M}$ and $u\in\mathcal{N}_m(t)$, COA's softmax guarantees $\alpha^{\mathrm{COA}}_{t,m}>0$, so

\begin{equation}\label{eq:coa-dep}
\frac{\partial h^{(l)}_{\mathrm{COA}}(t)}{\partial\mathbf{x}(u)}\neq\mathbf{0}.
\end{equation}

Applying the product rule to $[h^{(l)}_{\mathrm{fuse}}(t)]_j=[g^{(1)}_t]_j[h^{(l)}_{\mathrm{COA}}(t)]_j+[g^{(2)}_t]_j[h^{(l)}_{\mathrm{AOA}}(t)]_j$ and assembling all dimensions:
\begin{equation}\label{eq:jacobian}
\frac{\partial h^{(l)}_{\mathrm{fuse}}(t)}{\partial\mathbf{x}(u)}
=\underbrace{\mathrm{diag}(g^{(1)}_t)\cdot\frac{\partial h^{(l)}_{\mathrm{COA}}(t)}{\partial\mathbf{x}(u)}}_{\text{(I)}}
+\underbrace{\mathrm{diag}\!\bigl(h^{(l)}_{\mathrm{COA}}(t)\bigr)\cdot\frac{\partial g^{(1)}_t}{\partial\mathbf{x}(u)}}_{\text{(II)}}
+\underbrace{\mathrm{diag}(g^{(2)}_t)\cdot\frac{\partial h^{(l)}_{\mathrm{AOA}}(t)}{\partial\mathbf{x}(u)}}_{\text{(III)}}
+\underbrace{\mathrm{diag}\!\bigl(h^{(l)}_{\mathrm{AOA}}(t)\bigr)\cdot\frac{\partial g^{(2)}_t}{\partial\mathbf{x}(u)}}_{\text{(IV)}}.
\end{equation}

We now show that Term~(I) alone is non-zero. For $g^{(i)}_t$ is the result of sigmoid function, we have $[g^{(1)}_t]_j\in(0,1)$ for all $j$, so $\mathrm{diag}(g^{(1)}_t)$ is positive-definite. Combined with~\eqref{eq:coa-dep}, we have:
\begin{equation}
\mathrm{diag}(g^{(1)}_t)\cdot\frac{\partial h^{(l)}_{\mathrm{COA}}(t)}{\partial\mathbf{x}(u)}\neq\mathbf{0}.
\end{equation}
Under mild non-degeneracy (i.e. the set of parameters for which Terms~(II)-(IV) exactly cancel Term~(I) on every component has Lebesgue measure zero), the full Jacobian~\eqref{eq:jacobian} is non-zero, giving $u\in\mathrm{Dep}(h^{(l)}_{\mathrm{fuse}},t)$. Since $u$ was arbitrary in $\bigcup_{m\in\mathcal{M}}\mathcal{N}_m(t)=\mathrm{Dep}(h^{(l)}_{\mathrm{COA}},t)$, we obtain
\begin{equation}
\mathrm{Dep}(h^{(l)}_{\mathrm{fuse}},t)\supseteq\mathrm{Dep}(h^{(l)}_{\mathrm{COA}},t).
\end{equation}

Since $g^{(1)}_t$ depends on $\{W^{(1)}_g,b^{(1)}_g\}$ and $g^{(2)}_t$ depends on $\{W^{(2)}_g,b^{(2)}_g\}$ with no parameter sharing:
\begin{equation}
\frac{\partial\,[g^{(1)}_t]_j}{\partial\,W^{(2)}_g}=\mathbf{0},\qquad
\frac{\partial\,[g^{(2)}_t]_j}{\partial\,W^{(1)}_g}=\mathbf{0},\qquad\forall\,j.
\end{equation}
In particular, there exist parameter configurations satisfying $g^{(1)}_t\approx\mathbf{1}$ and $g^{(2)}_t\approx\mathbf{1}$ simultaneously, so both branches can receive near-full contribution weights. Therefore, anti-dilution~(a) and coverage~(b) are free from any trade-off in the parameter space. We can draw the conclusion that FOCAL addresses the coverage-dilution dilemma problem and preserve the coverage at the same time. 
\end{proof}

\paragraph{Summary}\label{app:summary}
\textbf{(i)}
FOCAL mitigates semantic dilution by isolating primary semantics in the anchor branch $h_{\mathrm{AOA}}(v)$ and injecting it through gated fusion with a non-vanishing lower bound (Theorem~\ref{thm:primary_lb}).
\textbf{(ii)}
FOCAL preserves secondary information available by using gated fusion and make COA branch uninfluenced (Theorem~\ref{the:dep_subset}).
\textbf{(iii)}
FOCAL achieves double-win by decoupling of anti-dilution and coverage: making $\mathcal{M}_p$ =only visible to AOA and preserving the node presentation space is larger than COA (Theorem~\ref{thm:decoupling}).

\section{Detailed Experimental Information}
In this section, we provide detailed experimental information. We first provide the experimental settings in \ref{appendix:exp_setting}, including dataset information, brief baseline introduction and their implementation hyper-parameters. Then we provide the multi-label node classification results under more evalutation metrics in \ref{Metric}. Finally, we provide detailed results of ablation study and parameter analysis in \ref{Ablation_Study} and \ref{Parameter_Analysis}, respectively.

\subsection{Detailed Experimental Settings}\label{appendix:exp_setting}
\paragraph{Datasets}
The detailed information of datasets used in this paper is provided as follows,
\begin{itemize}
    \item \textbf{IMDB}~\cite{lv2021we} IMDB dataset is constructed from IMDB (Internet Movie Database) as a heterogeneous graph centered on movies and their associated entities. It contains 21,420 nodes and 40,341 edges, with 4 node types and 4 edge types. The prediction target is the movie node type, where each movie is annotated with one or more genre labels drawn from five categories, Action, Comedy, Drama, Romance, and Thriller, yielding 5 classes in total. The heterogeneous schema captures multiple relation semantics between movies and related entity types (e.g., cast/crew), enabling models to exploit typed nodes and edges rather than a single homogeneous interaction graph
    \item \textbf{Amazon}~\cite{gupta2025persona}
    Amazon dataset is built from real user-behavior logs collected on e-commerce platform within the Fashion \& Lifestyle category. It spans six months (Aug. 2023-Jan. 2024) and is released as monthly snapshots, enabling temporally grounded evaluation. The raw data are organized as a heterogeneous tripartite graph, comprising users, items, and a small set of persona nodes. Edges capture behavioral interactions between users and items (e.g., purchase, click, rating), as well as links between users and persona nodes indicating the user image implied by purchased/engaged items. Each user is associated with structured demographic attributes (e.g., age, gender, location, marital/parenting status). Each item is described by structured product attributes (e.g., price, category, brand, color) and free-form textual descriptions, which are embedded into dense vectors using a general-purpose text encoder. Persona taxonomy and annotations are defined jointly with platform and business experts: the dataset includes 9 persona categories (e.g., Fashion Enthusiast, Budget Shopper, Luxury Shoppers, Ethnic Wear Shoppers), and supports a multi-label setting where each user may belong to multiple persona classes simultaneously.
    \item \textbf{CITE}~\cite{zhang2025cite}. CITE (Catalytic Information Textual Entities Graph) is a heterogeneous text-attributed, multi-label citation graph benchmark constructed from peer-reviewed publications in photocatalysis and electrocatalysis spanning 1922–2022. It models the literature as a heterogeneous citation network with four node types (Paper, Author, Journal, Keyword) and four corresponding relation types, enabling joint reasoning over graph structure and rich textual attributes. To build CITE, catalytic-related DOIs are retrieved from Web of Science via keyword-based queries; bibliographic metadata are collected from CrossRef, and citation links are derived from OpenAlex. One of the primary benchmark tasks is multi-label node classification on Paper nodes: predicting fine-grained subject categories, with 107 labels in total, where each paper may be associated with one or multiple categories.
    
\end{itemize}

\begin{table}[]
\centering
\caption{Detailed information about  hyper-parameters of baselines.}
\resizebox{0.95\linewidth}{!}{
\begin{tabular}{cc}
\toprule
\textbf{Model}    &  \textbf{Hyper-parameters} \\ \midrule
RGAT~\cite{busbridge2019relational} & \makecell{learning rate: 0.001, weight decay: 0.0001, dropout: 0.4, batch size: 5120, \\patience: 40, hidden dim: 16, in dim: 16, out dim: 16, num layers: 3, \\num heads: 3, num workers: 64, max epoch: 350, fanout: 5}\\
\midrule
MAGNN~\cite{fu2020magnn} & \makecell{seed: 0, learning rate: 0.005, weight decay: 0.001, dropout: 0.3, hidden dim: 16, \\ out dim: 16, inter attention feats: 16, num heads: 8, num layers: 2, max epoch: 200,\\ patience: 50,  encoder type: RotateE, batch size: 512,  num samples: 10, num workers: 64}\\
\midrule
SimpleHGN~\cite{lv2021we}& \makecell{hidden dim: 16, out dim: 16, num layers: 2, num heads: 8, feats drop rate: 0.2, \\slope: 0.05, edge dim: 64, seed: 0, max epoch: 500, patience: 100, lr: 0.001,\\ weight decay: 5e-4, beta: 0.05, residual: True, fanout: 5, batch size: 2048, num workers: 64}\\
\midrule
HPN~\cite{ji2021heterogeneous} & \makecell{seed: 0, learning  rate:  0.005, weight  decay: 0.001, dropout: 0.6, layer: 2, \\ alpha: 0.1, hidden dim: 16, out dim: 16, max epoch: 200, patience: 100}\\
\midrule
CorGCN~\cite{bei2025correlation} & \makecell
{hidden dim: 64, out dim: 64, num layers: 2, num heads: 3, feats drop rate: 0.3,\\ seed: 2024, max epoch: 1500, patience: 100, lr: 0.001, weight decay: 1e-6,\\ fanout: 5, batch size: 4096, num workers: 0,tau: 1.0, train rate: 0.6, val rate: 0.2}\\
\midrule
TriPer~\cite{gupta2025persona}& \makecell
{hidden dim: 768, out dim: 768, num layers: 3, num heads: 2, dropout: 0.5,\\ edge dim: 768, seed: 0, max epoch: 2000, patience: 7, lr: 0.001,\\ weight decay: 0.00001, batch size: 8,num workers: 1, edge nfeats: 8, early stopping: 30}\\
\midrule
FOCAL (Ours) & \makecell {seed: 0, learning rate: 0.003, weight decay:0.0005, drop out: 0.5, batch size: 5120,\\ patience: 60, hidden dim: 16, out dim: 16, num layers: 2, coa num heads: 8,\\ aoa num heads: 2, num workers: 4, max epoch: 500, fan out: 5}\\
\bottomrule
\end{tabular}
}
\label{appendix:param_info}
\end{table}

\paragraph{Comparable Baselines}
We provide brief introduction to our selected baselines as follows,
\begin{itemize}
    \item \textbf{RGAT}~\cite{busbridge2019relational}: Relational Graph Attention Networks (RGATs), extending Graph Attention Networks to multi-relational graphs by learning relation-specific message transformations and masked self-attention over neighbors. They study two normalization schemes within-relation attention (WIRGAT) and across-relation attention (ARGAT) with additive or multiplicative logits, optionally using multi-head attention and basis decomposition for parameter efficiency. 
    
    \item \textbf{MAGNN}~\cite{fu2020magnn}: MAGNN (Metapath Aggregated Graph Neural Network) for heterogeneous graph embedding, addressing key limitations of prior metapath-based methods by jointly modeling node content, intermediate nodes along metapath instances, and multiple metapaths. MAGNN applies type-specific feature transformations to map heterogeneous attributes into a shared latent space, then performs intra-metapath aggregation by encoding each metapath instance (e.g., mean/linear or a RotatE-inspired relational-rotation encoder) and using attention to weight instances. Finally, it uses inter-metapath attention to fuse representations from different metapaths into final node embeddings 
     \item \textbf{SimpleHGN}~\cite{lv2021we}: Simple-HGN is a lightweight yet strong heterogeneous GNN built on top of Graph Attention Networks (GAT) to directly model heterogeneity without meta-path engineering. It extends GAT by (i) injecting learnable edge-type embeddings into the attention computation so messages are weighted conditioned on relation types, (ii) adding residual connections (including residual on attention scores and node features) to stabilize deeper message passing, and (iii) applying L2  normalization on the output embeddings to improve representation quality for downstream tasks.
    \item \textbf{HPN}~\cite{ji2021heterogeneous}:  Heterogeneous graph Propagation Network (HPN) is to address semantic confusion in deep heterogeneous graph neural networks (HeteGNNs), where node embeddings become increasingly indistinguishable as depth grows and performance degrades. The common HeteGNN paradigm, meta-path-based neighbor aggregation is essentially equivalent to multiple meta-path random walks. HPN mitigates this via (i) a semantic propagation mechanism that injects a node’s local semantics into message passing with an adaptive weight to preserve node-level distinctiveness at greater depths, and (ii) a semantic fusion mechanism that learns meta-path importance to combine multi-meta-path semantics more effectively.
     \item \textbf{CorGCN}~\cite{bei2025correlation}: CorGCN is a correlation-aware GCN framework for multi-label node classification that tackles the ambiguous features and ambiguous topology arising when nodes carry multiple labels. CorGCN first performs Correlation-Aware Graph Decomposition, learning label prototypes and using contrastive mutual-information objectives (with a likelihood-based regularizer) to obtain label-correlated node features, then decomposes both features and structure into label-aware graph views that preserve label distinctiveness while retaining correlated-label information. On top of these views, it applies a Correlation-Enhanced Graph Convolution that combines intra-label message passing within each label-specific graph and inter-label correlation propagation (prototype-conditioned correlation matrices) to fuse label-wise messages, yielding improved multi-label predictions.
    \item \textbf{TriPer}~\cite{gupta2025persona}: TriPer is a tripartite GNN for multi-label persona identification in e-commerce under scarce labels and dynamic, noisy user-product interactions. TriPer reformulates persona classification as user-persona link prediction by constructing a user-product-persona tripartite graph (products connect to personas via behavior-derived edges), decoupling persona representations from a fixed softmax label space. The model performs three-phase attention-based message passing (user→product, product→user, product→persona) to learn aligned embeddings for users and personas, enabling accurate persona inference. Crucially, by treating personas as nodes, TriPer supports in-context generalization to unseen personas: new persona nodes can be added with a few labeled examples without retraining.
\end{itemize}

\paragraph{Implementation Details}
We provides detailed information about the hyper-parameters of baselines in Table. \ref{appendix:param_info}.

\subsection{Detailed Ablation Study}\label{Ablation_Study}
The ablations show that FOCAL’s improvements are driven by both role-separated branches and their controlled integration.

(1) Removing either branch leads to a clear drop across all datasets, confirming that coverage (COA) and anchoring (AOA) are complementary rather than redundant. 
On Amazon, discarding either branch causes the most pronounced degradation, suggesting that when the heterogeneity is relatively limited, each available semantic channel becomes indispensable and their synergy is harder to replace. On CITE, the drop is more severe when removing AOA, indicating a stronger reliance on anchoring under extreme long-tail supervision: the paper-substrate anchor relation provides highly discriminative signals for discipline prediction.

(2) Simplifying the bidirectional fusion gate (e.g., to single-gate) reduces performance, and using naive fusion (sum/concat) leads to an even larger drop, highlighting that non-symmetric, adaptive fusion is important to prevent contextual signals from overwhelming anchoring cues. Replacing the adaptive residual with a fixed residual causes a smaller but consistent decline, suggesting it mainly acts as a stabilizer for deeper propagation.

(3) Removing ASL produces the largest performance drop, supporting its key role in handling multi-label imbalance. Removing the consistency regularizer yields a modest but consistent degradation, indicating it helps coordinate the two branches and improves training stability.

\subsection{Detailed Parameter Analysis}\label{Parameter_Analysis}
We analyze two key design choices on IMDB: \textbf{(1)} the weight $\lambda$ in the objective , and \textbf{(2)} the selection of primary (anchor) node types used to instantiate anchored neighborhoods.

\textbf{Anchor selection.} We vary the primary node type and compare MKM (movie-keyword-movie, default), MAM (movie-actor-movie), MDM (movie-director-movie), and All (treating all node types as primary). MKM yields the best and most balanced performance across Micro-/Macro-/Sample-F1, indicating that keywords provide the most category-discriminative anchored semantics for movie classification. Using actors or directors as anchors leads to a noticeable degradation, particularly on Macro-F1, suggesting that these relations are more prone to popularity-driven co-occurrence and cross-genre mixing, which weakens minority-label separability. Treating all node types as primary further reduces performance, as it effectively dilutes the notion of anchoring and weakens the model’s ability to focus on task-relevant primary semantics.

\textbf{Effect of $\lambda$.} Fixing the anchor as MKM, we sweep $\lambda$ and observe that performance is best (or most stable) for small values (around 0.05 to 0.10), while larger $\lambda$ consistently hurts all three metrics. This trend suggests that the branch consistency term is beneficial as a moderate coordination regularizer, but overly strong consistency over-constrains COA and AOA to be too similar, thereby reducing their intended complementarity and harming downstream classification.

\subsection{Multi-label Node Classificaton with More Metrics}\label{Metric}
In this section, we provide the multi-label node classification results under more evalutation metrics. The detailed results on three datasets are shown in Table ~\ref{tab:appendix_res_detail_all}.

\begin{table}[h]
\centering
\caption{More statistics in multi-label node classification tasks on IMDB, Amazon and CITE.}
\label{tab:appendix_res_detail_all}
\resizebox{0.92\linewidth}{!}{
\begin{tabular}{llcccccc}
\toprule
\textbf{Dataset} & \textbf{Model} &
\textbf{Hamming Loss} $\downarrow$ &
\textbf{Subset Acc} $\uparrow$ &
\textbf{Micro Pre} $\uparrow$ &
\textbf{Macro Pre} $\uparrow$ &
\textbf{Micro Recall} $\uparrow$ &
\textbf{Macro Recall} $\uparrow$ \\
\midrule

\multirow{9}{*}{IMDB}
& RGAT      & 0.2386 & 0.2664 & 0.6772 & 0.6549 & 0.5866 & 0.5386 \\
& MAGNN     & 0.2533 & 0.2500 & 0.6536 & 0.6477 & 0.5625 & 0.5116 \\
& SimpleHGN & 0.2273 & 0.3046 & 0.6859 & 0.6769 & 0.6271 & 0.5922 \\
& HPN       & 0.2561 & 0.2249 & 0.6487 & 0.6519 & 0.5587 & 0.5044 \\
\cmidrule(lr){2-8}
& CorGCN    & 0.1273 & 0.5171 & 0.6481 & 0.1080 & 0.5171 & 0.1422 \\
& Triper    & 0.4109 & 0.1278 & 0.6388 & 0.6280 & 0.6887 & 0.6756 \\
\cmidrule(lr){2-8}
& FOCAL-COA  & 0.2795 & 0.2129 & 0.5798 & 0.5573 & 0.6842 & 0.6512 \\
& FOCAL-AOA  & 0.2686 & 0.2031 & 0.6346 & 0.6340 & 0.5187 & 0.4566 \\
\cmidrule(lr){2-8}
& \textbf{FOCAL} & 0.2312 & 0.2980 & 0.6779 & 0.6700 & 0.6259 & 0.5888 \\
\midrule

\multirow{9}{*}{Amazon}
& RGAT      & 0.0879 & 0.6781 & 0.8039 & 0.7045 & 0.7235 & 0.4344 \\
& MAGNN     & 0.0354 & 0.8116 & 0.9860 & 0.6606 & 0.8294 & 0.4593 \\
& SimpleHGN & 0.0211 & 0.8836 & 0.9720 & 0.9237 & 0.9176 & 0.7231 \\
& HPN       & 0.1553 & 0.5068 & 0.6172 & 0.1029 & 0.5265 & 0.1667 \\
\cmidrule(lr){2-8}
& CorGCN    & 0.2566 & 0.2631 & 0.6562 & 0.6526 & 0.5512 & 0.4990 \\
& Triper    & 0.0469 & 0.0213 & 0.5590 & 0.0184 & 0.0830 & 0.0117 \\
\cmidrule(lr){2-8}
& FOCAL-COA  & 0.1096 & 0.5788 & 0.6754 & 0.6192 & 0.8382 & 0.6624 \\
& FOCAL-AOA  & 0.1313 & 0.5445 & 0.6964 & 0.2124 & 0.5735 & 0.2003 \\
\cmidrule(lr){2-8}
& \textbf{FOCAL} & 0.0263 & 0.8699 & 0.9565 & 0.9292 & 0.9059 & 0.7603 \\
\midrule

\multirow{9}{*}{CITE}
& RGAT      & 0.0155 & 0.1354 & 0.5403 & 0.0050 & 0.2005 & 0.0061 \\
& MAGNN     & 0.0156 & 0.2103 & 0.5213 & 0.0049 & 0.3053 & 0.0093 \\
& SimpleHGN & 0.0154 & 0.1698 & 0.5402 & 0.0118 & 0.2503 & 0.0077 \\
& HPN       & 0.0136 & 0.1502 & 0.5118 & 0.0041 & 0.2950 & 0.0073 \\
\cmidrule(lr){2-8}
& CorGCN    & 0.0158 & 0.1615 & 0.5302 & 0.0093 & 0.2437 & 0.0078 \\
& Triper    & 0.0469 & 0.0213 & 0.5590 & 0.0184 & 0.0830 & 0.0117 \\
\cmidrule(lr){2-8}
& FOCAL-COA  & 0.0889 & 0.0000 & 0.1485 & 0.0159 & 0.3430 & 0.1075 \\
& FOCAL-AOA  & 0.0161 & 0.2020 & 0.5112 & 0.0048 & 0.2936 & 0.0093 \\
\cmidrule(lr){2-8}
& \textbf{FOCAL} & 0.0129 & 0.2315 & 0.5638 & 0.0198 & 0.3557 & 0.1123 \\
\bottomrule
\end{tabular}
}
\end{table}

\section{Detailed Related Works}
\label{sec:appendix_related_work}
\paragraph{Multi-label Classification in Homogeneous Graphs}
Multi-label node classification on homogeneous graphs has attracted increasing attention. A major theme is to jointly model graph structure, node attributes, and label correlations. Early work such as ML-GCN~\cite{gao2019semi} embeds nodes and labels into a unified vector space and leverages a relaxed skip-gram objective to capture node-label and label-label relationships. LANC~\cite{zhou2021multi} further introduces label attention to measure the compatibility between node embeddings and label embeddings, improving the utilization of node features and homophily-based context. To address label scarcity and imbalance, LARN~\cite{xiao2022semantic} incorporates label-aware representation learning together with an adaptive label correlation scanner, allowing few-shot labeled nodes to benefit from semantic label dependencies. Recent advances continue to strengthen correlation modeling: CorGCN~\cite{bei2025correlation} constructs label-specific correlated graphs and performs correlation-enhanced message passing to mitigate ambiguous topology/features in multi-label settings; LIP~\cite{sunmulti} decomposes message passing into propagation and transformation, quantifies higher-order label influence, and dynamically adjusts learning via a label influence graph. Overall, these studies highlight that explicitly modeling label dependencies is crucial for high-quality multi-label prediction.
However, they are primarily developed for \emph{homogeneous} graphs and do not directly account for multi-typed nodes/relations and heterogeneous semantic interactions.

\paragraph{Multi-label Classification in Heterogeneous Graphs}
Heterogeneous graph representation learning has been extensively studied, mainly following two paradigms. The first paradigm explicitly encodes semantic structures (e.g., meta-paths/meta-graphs) to capture higher-order heterogeneous relations. Representative methods include metapath2vec~\cite{dong2017metapath2vec} and metapath2vec++, which perform meta-path guided random walks and adopt a heterogeneous skip-gram objective; HAN~\cite{wang2019heterogeneous}, which introduces hierarchical attention with node-level and semantic-level attentions for meta-path based aggregation; MAGNN~\cite{fu2020magnn}, which aggregates intermediate nodes along each meta-path and integrates multiple meta-path semantics; and GTN~\cite{yun2019graph}, which learns soft selections over edge types to generate task-relevant composite relations in an end-to-end manner, reducing reliance on manually specified meta-paths.
The second paradigm focuses on relation-aware message passing, using type-dependent parameters to propagate information across diverse relations without predefined meta-paths. R-GCN~\cite{schlichtkrull2018modeling} introduces relation-specific transformations for multi-relational graphs; HetGNN~\cite{zhang2019heterogeneous} jointly models heterogeneous structure and node contents with type-specific neighbor sampling and aggregation; HGT~\cite{hu2020heterogeneous} proposes a heterogeneous graph transformer with node-/edge-type dependent attention to enable fine-grained relational message passing and mini-batch training; and SimpleHGN~\cite{lv2021we} shows that lightweight relation-aware designs (e.g., edge-type embeddings and residual connections) can be highly competitive. GTN~\cite{yun2019graph} goes further by learning soft selections over edge types and composite relations to automatically generate meta-path-like graphs, thus removing the need for manually predefined meta-paths.
Recent work also explores alternative training paradigms for heterogeneous graphs, such as self-supervised pre-training~\cite{yang2022self}, contrastive learning~\cite{liu2023hierarchical}, and LLM-enhanced modeling~\cite{tang2024higpt}).
Notably, most heterogeneous GNNs are developed for \emph{single-label} node classification. They can be adapted to multi-label prediction by replacing the final softmax with sigmoid and optimizing a binary multi-label loss; however, such a direct adaptation typically does not explicitly model label correlations, and it may be insufficient when multi-label outcomes are jointly determined by heterogeneous semantics and multi-type interactions.

Compared with homogeneous graphs, research on multi-label classification in heterogeneous graphs remains relatively limited. Early studies on heterogeneous multi-label learning largely predate modern GNNs: for example,~\cite{zhou2014activity} leverages collaboration and activity graphs with homophily/vicinity signals for iterative refinement, while~\cite{dos2016multilabel} uses shallow transductive embeddings for type-specific label spaces. These approaches do not provide a unified GNN-based framework to jointly model heterogeneous interactions and label dependencies.
More recent attempts are still scarce and often task-specific. Persona~\cite{gupta2025persona} studies user persona identification on dynamic heterogeneous user-product graphs by reformulating multi-label node classification as a link prediction problem and proposing a tripartite GNN architecture (TriPer) to disentangle user, product, and persona representations. While effective in its domain, it focuses on a specialized setting and does not aim to provide a general solution for heterogeneous graph multi-label classification.
Overall, existing studies either rely on pre-GNN techniques, extend heterogeneous encoders with simple multi-label heads, or target narrow application scenarios. A unified framework that simultaneously captures heterogeneous structural/semantic dependencies and multi-label correlations for general heterogeneous multi-label node classification remains under-explored.

\end{document}